\RequirePackage{fix-cm}
\documentclass[smallextended]{svjour3}       
\smartqed  

\usepackage[latin1]{inputenc}
\usepackage[english]{babel}
\usepackage{amsmath}
\usepackage{amsfonts}
\usepackage{amssymb}
\usepackage{color}
\usepackage{hyperref} 
\usepackage{graphicx}
\usepackage{caption}
\newcommand{\PP}{\mathbb{P}}

\newcommand{\RR}{\mathbb{R}}

\newcommand{\CC}{\mathbb{C}}
\newcommand{\NN}{\mathbb{N}}
\newcommand{\B}[1]{\mathbf{#1}}

\newcommand{\om}{\mathbf{\Omega}}

\newcommand{\MM}{\mathbf{M}}

\newcommand{\supp}{\mathrm{supp}}
\newcommand{\proj}{\mathrm{proj}}
\newcommand{\rank}{\mathrm{rank}}
\newcommand{\HF}{\mathrm{HF}}
\newcommand{\lgl}{<_{gl}}
\newcommand{\leqgl}{\leq_{gl}}

\def\x{\mathbf{x}}
\def\y{\mathbf{y}}

\def\B{\mathbf{B}}

\def\R{\mathbb{R}}
\def\om{\mathbf{\Omega}}

\def\p{\mathbf{p}}
\def\q{\mathbf{q}}
\def\v{\mathbf{v}}
\def\u{\mathbf{u}}
\def\z{\mathbf{z}}
\def\e{\mathbf{e}}
\def\a{\mathbf{a}}
\def\umeas{\lambda}

\newtheorem{defn}[theorem]{Definition}
\newtheorem{assumption}[theorem]{Assumption}

\begin{document}

\title{Data analysis from empirical moments and the Christoffel function}



\author{Edouard Pauwels	\and Mihai Putinar \and Jean-Bernard Lasserre}


\institute{Edouard Pauwels \at
							IRIT, CNRS, AOC, Universit\'e de Toulouse. DEEL, IRT Saint Exup\'ery, France.
              \email{epauwels@irit.fr}. Phone contact: +33561556351.           
           \and
           Mihai Putinar \at
              University of California, Santa Barbara, US, \  and \ Newcastle University, Newcastle upon Tyne, UK.
           \and
           Jean-Bernard Lasserre \at
           		LAAS-CNRS and Institute of Mathematics, University of Toulouse, LAAS, 7 avenue du Colonel Roche 31077 Toulouse, France.
}

\date{Received: date / Accepted: date}

\maketitle

Communicated by Felipe Cucker.\\

\begin{abstract}
				Spectral features of the empirical moment matrix constitute a resourceful tool for unveiling properties of a cloud of points, among which, density, support and latent structures. This matrix is readily computed from an input dataset, its eigen decomposition can then be used to identify algebraic properties of the support or density / support estimates with the Christoffel function. It is already well known that the empirical moment matrix encodes a great deal of subtle attributes of the underlying measure. Starting from this object as base of observations we combine ideas from statistics, real algebraic geometry, orthogonal polynomials and approximation theory for opening new insights relevant for Machine Learning (ML) problems with data supported on algebraic sets. Refined concepts and results from real algebraic geometry and approximation theory are empowering a {\em simple} tool (the empirical moment matrix) for the task of solving non-trivial questions in data analysis. We provide (1) theoretical support, (2) numerical experiments and, (3) connections to real world data as a validation of the stamina of the empirical moment matrix approach.
\end{abstract}

\section{Introduction}
\subsection*{Inference of low dimensional structures}
Intrinsic dimension has a long history in signal processing expressing the idea that most empirical high dimensional signals are actually structured and can be approximated by a small number of entities \cite{bennet1969intrinsic,fukunaga1971algorithm,pettis1979intrinsic,trunk1968statistical}. Dimension reduction in data analysis has witnessed a considerable renewal of interest in the early 2000's with the advent of non linear low dimensional structure estimation algorithms \cite{roweis2000nonlinear,tenenbaum2000global,belkin2002laplacian} and follow up works \cite{donoho2003hessian,law2006incremental,elhamifar2011sparse}. Graph Laplacian based methods were proposed in \cite{belkin2005toward,hein2005graph} and intrinsic dimension estimation was revisited in \cite{levina2005maximum,hein2005intrinsic}. In a manifold learning context the questions of finite sample efficiency were treated in \cite{genovese2012minimax,genovese2012manifold,kim2015tight,aamari2018nonasymptotic} and a statistical test was proposed \cite{fefferman2016testing}. A spectral support estimator was described in \cite{devito2014learning}.

Getting access to topological properties of a data distribution through computational topology is of rising interest \cite{niyogi2008homology,edelsbrunner2008peristent,carlsson2009topology,ghrist2008barcodes}. This was cast in a statistical framework in \cite{chazal2011geometric,chazal2015convergence,bubenik2015statistical} and in a machine learning framework by \cite{niyogi2008homology,niyogi2011topological}. Our approach is distinct and rather complementary to these studies, by accepting from the very beginning some natural algebraic constraints on the random variables. To be more specific, we deal in a couple of examples included in this work with pairs of angles belonging to a torus, or a direction lying on a sphere, very close to familiar settings in classical mechanics.

Real algebraic geometry \cite{bochnak1998real} and its related computational tools \cite{cox2007ideals} constitute a well established mathematical branch dealing with varieties described by polynomials. Our work streams from the potential of exploiting these results with the specific aim at inferring latent structures. A fundamental view in modern algebraic geometry is to study a finite dimensional set $S$ via duality, by investigating the algebra of polynomial functions acting on $S$. We put at work the idea to capture geometric characteristics of $S$ from polynomials on a cloud of points spread across $S$, essentially a spectral property of the corresponding empirical moment matrix. Our approach is closely related to elements mentioned in \cite{breiding2018learning} which focuses on computational aspects, while we tackle statistical issues arising when working with finite samples.

\subsection*{Christoffel-Darboux kernel}
Given a positive, rapidly decaying Borel measure $\mu$ on euclidean space, integration with respect to $\mu$ defines a scalar product on the space of polynomials. The reproducing kernel associated to the Hilbert space of polynomials up to a given degree is called the Christoffel-Darboux kernel \cite{simon2008christoffel} and is usually computed from a suitable orthonormal family \cite{szego1974orthogonal,dunkl2001orthogonal}. This kernel only depends on the moments of $\mu$ and has been intensively studied for more than a century. It captures refined properties of the original measure \cite{nevai1986geza,van1987asymptotics} and allows to decode information concerning the support, the density of the absolutely continuous part of $\mu$ and more importantly it can detect the presence of a singular continuous component of $\mu$. The univariate case was treated in \cite{mate1980bernstein,mate1991szego,totik2000asymptotics}, while a recent application to dynamical systems is given in \cite{korda2017data}. 

In the multivariate setting, explicit computations are known only for simple supporting varieties \cite{bos1994asymptotics,xu1996asymptotics,bos1998asymptotics,xu1999asymptoticsSimplex} or under abstract assumptions \cite{kroo2012christoffel}. As claimed in \cite{nips,arxiv}, these tools constitute a promising research direction for data analysis. 
For latent structure inference one needs to understand properties of these objects in the \emph{singular} setting, hardly tackled in the literature, the main goal of the present paper is to tackle this issue for measures supported on algebraic sets. We refere to this situation as {\em singular} throughout the text since it results in rank deficiency of the moment matrix.

\subsection*{Results}
(i) Under natural assumptions, we show that the space of polynomials on a real algebraic set and the space of polynomials on a large enough generic sample on the same set are essentially the same spaces. As a result, global geometric properties of the support can be inferred \emph{with probability one} from only finite samples. This relies on the rigidity of the polynomial setting, and constitutes a significant departure from more classical forms of statistical inference which most often hold non deterministically. In particular, we describe
how the growth of the rank of the \emph{empirical} moment matrix is related to the dimension of the support of the underlying measure. Again, with probability one, the intrinsic dimension of the set of data points can be inferred exactly from the empirical moment matrix. 

(ii) Our second main result extends the weighted asymptotic convergence of the Christoffel function to the density of the underlying measure, cf. \cite{kroo2012christoffel}. The work of \cite{kroo2012christoffel} relies on the variational formulation of the Christoffel function, still valid in the singular case. Starting from simple assumptions, we cover the cases of the sphere, the ellipsoids and canonical operations of them, such as fibre products. The proof provides \emph{explicit} convergence rates in supremum norm,  which to the best of our knowledge, is the first result of this type in the multivariate setting.

Both results heavily rely on the fact that we have access to sample points on a lower dimensional algebraic set. In many practical situations, noise may lead to measurements close to such a set, but not exactly on this set. We argue that there are intrinsic connections between both situations using a continuity argument. We believe that understanding the more general noisy setting requires to understand the, sometimes ideal, noiseless setting. Our work constitutes a mandatory first step to handle such complex situations. Furthermore, we consider practical applications for which the data do live on a lower dimensional algebraic set by construction, in this setting our approach completely captures the difficulty of the problem. 

\subsection*{Numerical experiments} We illustrate on simulations that the growth of the rank of the empirical moment matrix is related to the dimensionality of the support of the underlying measure. Connection with the density is illustrated on three real world datasets featuring symmetry and periodicity. We map them to higher dimensional algebraic sets which captures naturally the corresponding symmetries and deploy the Christoffel function machinery illustrating how this simple tool can be used to estimate densities on the circle, the sphere and the bi-torus. Experiment details are found in the appendix.

\subsection*{Organisation of the paper}
Section \ref{sec:nonSingular} presents the Christoffel-Darboux kernel for positive definite moment matrices, this is ensured for example if the underlying measure is absolutely continuous. The singular case is tackled in Section \ref{sec:singular} where we provide a general construction and first describe its geometric features. Section \ref{sec:finiteSampleApprox} exposes our main results on approximation of Christoffel-Darboux kernels from finite samples and Section \ref{sec:sphere} describes our main result about relation with the underlying density with respect to a singular reference measure. Numerical experiments are presented in Section \ref{sec:numerics} and a discussion about slightly perturbed singular measures is given in Section \ref{sec:discussion}.
All proof arguments are postponed to Section \ref{sec:proofs}; the Appendix contains additional details on notation, numerical experiments and some technical Lemmas. 

The rather odd structure of the article was dictated by our primary aim of making accessible the main results to a wider audience. In this respect, some technical proofs and references to higher mathematics can be omitted at a first reading. On the other hand experts in algebraic geometry or approximation theory will find complete and sometimes redundant indications of the proofs of all claimed statements.

\section{The non-degenerate case}
\label{sec:nonSingular}
Henceforth we fix the dimension of the ambient euclidean space to be $p$. For example, we will consider vectors in $\RR^p$ as well as $p$-variate polynomials with real coefficients. 
\subsection{Notations and definition}
\label{sec:notations}

We denote by $\RR[X]$ the algebra of $p$-variate polynomials with real coefficients. 
For $d \in \NN$, $\RR_d[X]$ stands for the set of $p$-variate polynomials of degree less than or equal to $d$. We set $s(d) = {p+d \choose d}  = \dim \RR_d[X]$. 

Let $\v_d(X)$ denotes the vector of monomials of degree less than or equal to $d$.
A polynomial $P\in\RR_d[X]$ can be written as $P(X) = \left\langle\p, \v_d(X)\right\rangle$ for a given $\p  \in \RR^{s(d)}$.  
For a positive Borel measure $\mu$ on $\R^p$ denote by ${\rm supp}(\mu)$ its support, i.e., the smallest closed subset whose complement has measure zero.
We will only consider rapidly decaying measures, that is measures whose all moments are finite.

\paragraph{Moment matrix}
For any $d \in \NN$, the moment matrix of $\mu$, with moment up to $2d$, is given by 
$$\MM_{\mu,d} = \int_{\RR^p} \v_d(\x) \v_d(\x)^T d\mu(\x)$$ 
where the integral is understood element-wise. Actually it is useful to interpret the moment matrix as representing the bilinear form on $\RR[X]$, $\left\langle \cdot, \cdot\right\rangle_\mu \colon (P,Q) \mapsto \int PQd\mu$,
restricted to polynomials of degree up to $d$. Indeed, if $\p,\q \in \RR^{s(d)}$ are the vectors of coefficients of any two polynomials $P$ and $Q$ of degree up to $d$, one has $\p^T \MM_{\mu,d} \q =  \left\langle P, Q\right\rangle_\mu$.
This entails that $\MM_{\mu,d}$ is positive semidefinite for all $d \in \NN$.

\subsection{Christoffel-Darboux kernel and Christoffel function}
\label{sec:constructionAbsCont}
In this section we assume that the probability measure $\mu$ is not supported by a proper real algebraic subset of $\RR^p$. 
In other words, for any polynomial $p$:
$$ \int p^2 d\mu = 0 \ \ {\rm if \ and \ only \ if} \ \ p =0.$$
This non-degeneracy condition is assured for instance for an absolutely continuous measure with respect to Lebesgue measure on $\RR^p$.

Fix $d$ a positive integer. In this case, the bilinear form $\left\langle \cdot, \cdot\right\rangle_\mu$ is positive definite on $\RR_d[X]$ which is a finite dimensional Hilbert space; we denote by $\|\cdot\|_\mu$ the corresponding norm. For any $\x \in \RR^p$, the evaluation functional on $\RR_d[X]$, $P \mapsto P(\x)$, is continuous with respect to $\|\cdot\|_\mu$. Hence $\RR_d[X]$ is a Reproducing Kernel Hilbert Space (RKHS), it admits a unique reproducing kernel, $\kappa_{\mu,d}$ \cite{aronszajn1950theory}, called the {\it Christoffel-Darboux kernel} \cite{simon2008christoffel}. 
The reproducing property and definition of $\MM_{\mu,d}$ ensure that for any $\x,\y \in \RR^p$,
\begin{align}
				\kappa_{\mu,d}(\x,\y) = \v_d(\x)^T \MM_{\mu,d}^{-1} \v_d(\y).
				\label{eq:christoffelDarboux}
\end{align}
The {\it Christoffel function} is defined as follows:
\begin{align}
				\Lambda_{\mu,d} \colon \x \mapsto \inf_{P \in \RR_d[X]} \left\{ \int P(\z)^2 d\mu(\z), \quad\mathrm{ s.t. }\quad P(\x) = 1 \right\} 
				\label{eq:christoffel}
\end{align}
A well known crucial link between the Christoffel-Darboux kernel and the Christoffel function is given by the following Lemma of which we give a short proof for completeness.
\begin{lemma}
				For any $\x \in \RR^p$, one has
				\begin{align}
								\kappa_{\mu,d}(\x,\x) \Lambda_{\mu,d}(\x) = 1,
								\label{eq:fundamentalIdentity}
				\end{align}
				and the solution of \eqref{eq:christoffel} is given by $P^*\colon \z \mapsto \frac{\kappa_{\mu,d}(\x,\z)}{\kappa_{\mu,d}(\x,\x)}$.
				\label{lem:fundamentalIdentity}
\end{lemma}
\begin{proof}[Proof of Lemma \ref{lem:fundamentalIdentity}]
				If $P$ is feasible for \eqref{eq:christoffel},  then $\int P(\x) \kappa_{\mu,d}(\z,\x) d\mu(\x) = P(\z) = 1$. Cauchy-Schwartz inequality yields
				\begin{align*}
								1 \leq \int (P(\x))^2  d\mu(\x) \int (\kappa_{\mu,d}(\z,\y))^2 d\mu(\y) = \kappa_{\mu,d}(\z,\z) \int (P(\x))^2  d\mu(\x).
				\end{align*}
				Choosing $P = P^*$, we obtain equality in the above inequality which proves the desired result.
\end{proof}
\section{The singular case}
\label{sec:singular}
Throughout this section $\mu$ denotes a Borel probability measure possessing all moments. The word {\em singular} refers to rank deficiency of the moment matrix, a consequence of the support of $\mu$ being a subset of an algebraic set.
\subsection{Departure from the regular case}
Relations \eqref{eq:christoffelDarboux} and  \eqref{eq:fundamentalIdentity} do not need to hold in the singular case as the bilinear form $\left\langle \cdot,\cdot\right\rangle_\mu$ may fail to be positive definite. In this situation the definition of polynomial bases requires additional care and the RKHS interpretation is not valid over the whole space $\RR^p$. There are essentially two ways to circumvent this difficulty:
\begin{itemize}
				\item Construct, using a quotient map, the space of polynomials restricted on $\supp(\mu)$, define the kernel based on this family and extend it to the whole euclidean space.
				\item Consider only the variational formulation as in the left hand side of \eqref{eq:christoffel}.
\end{itemize}
These are equivalent on the support of $\mu$ but do not lead the same definition outside of the support. Identity \eqref{eq:fundamentalIdentity} ceases to hold in the singular case and both constructions provide a valid extension outside of the support. We will focus on the second choice and start by stating the following basic Lemma. 
\begin{lemma}
				Let $p\in\NN^*$, $\MM \in \RR^{p\times p}$ be symmetric semidefinite and $\u \in \RR^p$. Let $\MM^\dagger$ denotes the Moore-Penrose pseudo inverse, we have
				\begin{align*}
								\inf_{\x \in \RR^p} \left\{\x^T \MM \x; \ \x^T\u = 1    \right\} \qquad=\qquad 
								\begin{cases}
												\qquad\frac{1}{\u^T \MM^\dagger \u},\  \text{if } \quad\proj_{\ker(\MM)} (u) = 0,\\
												\qquad 0,\, \text{otherwise}\\
								\end{cases}
				\end{align*}
				\label{lem:pseudoInv}
\end{lemma}

\subsection{Support of the measure}

If $\MM_{\mu,d}$ does not have full rank, then this means that its support is contained in an algebraic set. 
We let $V$ be the Zariski closure of $\supp(\mu)$, that is the smallest real algebraic set which contains $\supp(\mu)$. It is well known that
an algebraic set is always equal to the common zero set of finitely many polynomials. Note that in the present article we exclusively work over the real field of coefficients, a setting which is more intricate than that offered by an algebraically closed field (for instance Hilbert's Nullstellensatz is not valid in its original form over a real field).
\begin{lemma}
				Denote by $\mathcal{I}$ the set of polynomials $P$ which satisfy $\int P^2 d\mu = 0$. 
				The set $\mathcal{I}$ coincides with the ideal of polynomials vanishing on $V$.
				\label{lem:radical}
\end{lemma}
We conclude this section by underlying important relations between the moment matrix and geometric properties of $V$. The proof combines \cite[Proposition 2]{laurent2012approach} and a well known identification of the dimension of a variety in the leading term of the associated  Hilbert's polynomial \cite[Chapter 9]{cox2007ideals}. 
An illustration of the relation between the growth of the rank of the moment matrix and the dimension of the underlying set is given in Section \ref{sec:numerics} with numerical simulations.
\begin{proposition}
				For all $d \in \NN^*$, $\rank(\MM_{\mu,d}) = \HF(d)$, where $\HF$ denotes the Hilbert function of $V$ and gives the dimension of the space of polynomials of degree up to $d$ on $V$. For $d$ large enough, $\rank(\MM_{\mu,d})$ is a polynomial in $d$ whose degree is the dimension of $V$ and $\ker(\MM_{\mu,d})$ provides a basis generating the ideal $\mathcal{I}$.
				\label{prop:idealRank}
\end{proposition}

\subsection{Construction of the Christoffel-Darboux kernel}
We denote by $L_d^2(\mu)$ the space of polynomials on $V$ of degree up to $d$ with inner product and norm induced by $\mu$ denoted by $\langle\cdot,\cdot\rangle_\mu$ and $\|\cdot\|_\mu$ respectively. This quotient space can be obtained by identifying two polynomials which are equal on $V$. More specifically, $L_d^2(\mu)$ is the quotient space $\RR_d[X] / (\mathcal{I}\cap \RR_d[X])$, and it can be easily checked that $\langle\cdot,\cdot\rangle_\mu$ is positive definite on $L_d^2(\mu)$ and hence induces a genuine scalar product. 

As in the absolutely continuous setting, pointwise evaluation is continuous with respect to $\|\cdot\|_\mu$ and $L_d^2(\mu)$ is a RKHS \cite{aronszajn1950theory}. The symetric reproducing kernel is defined for all $\x \in V$, by $\kappa_{\mu,d}(\cdot,\x) \in L_d^2(\mu)$ such that for all $P \in L_d^2(\mu)$,
\begin{align*}
				\int P(\y) \kappa_{\mu,d}(\y,\x) d\mu(\y) = P(\x).
\end{align*}
We consider the variational formulation of the Christoffel function in \eqref{eq:christoffel}.
Considering the restriction of $\Lambda_{\mu,d}$ to $V$ amounts to replace $\RR_d[X]$ by $L_d^2(\mu)$, so that for all $\z \in V$
\begin{equation}
\label{def-christo-d}
\Lambda_{\mu,d}(\z)\,=\,\displaystyle\inf_{P\in L_d^2(\mu)}\,\left\{\,\int (P(\x))^2\,d\mu(\x):\: P(\z)\,=\,1\,\right\}.
\end{equation}
Note that \eqref{eq:fundamentalIdentity} still holds for all $\z \in V$, $\Lambda_{\mu,d}(\z) \kappa_{\mu,d}(\z,\z) = 1$.

By definition of the Zariski closure, a polynomial vanishing on $\supp(\mu)$ also vanishes on $V$ and we deduce that:
\begin{align}
				\label{eq:nonZeroV}
				\left\{ \z \in \RR^p,\, \Lambda_{\mu,d}(\z) = 0 \right\} \cap V = \emptyset.
\end{align}
Therefore if $\z\in V$ then $\Lambda_{\mu,d}(\z)>0$ and Lemma \ref{lem:pseudoInv} ensures that $\Lambda_{\mu,d}(\z)\,=\,(\v_d(\z)^T\MM_{\mu,d}^\dag\,\v_d(\z))^{-1}$.
Since $V$ is algebraic, it is the solution set of a finite number of polynomial equations. Let $d_V$ be the maximum degree of such equations.
Then for any $d \geq d_V$: $V = \left\{ \x \in \RR^p:\: \Lambda_{\mu,d}(\x) > 0 \right\}$.

\section{Finite sample approximation}
\label{sec:finiteSampleApprox}
Given a positive Borel measure $\mu$ one may consider a finite sample $\left\{ \x_i \right\}_{i=1}^n$ of independant random vectors drawn from $\mu$ and replace $\mu$ by its empirical counterpart $\mu_n = \frac{1}{n} \sum_{i=1}^n \delta_{\x_i}$ where $\delta_\a$ denotes the Dirac measure at $\a$. In this situation, for $i = 1,\ldots, n$, $\x_i$ is a random variable, hence $\mu_n$ is a random measure, $\MM_{\mu_n,d}$ is a random matrix, and the law of large numbers states that
for every fixed $d$,  $\|\MM_{\mu_n,d} - \MM_{\mu,d}\| \to 0$ almost surely.
This entails that $\rank(\MM_{\mu_n,d}) \to \rank(\MM_{\mu,d})$ almost surely. Classical approches to strengthen this result and provide quantitative estimates rely on concentration of measure for matrices or operators. In the present section, we investigate another direction and provide a different view of this convergence phenomenon.
{\em Under suitable assumptions, the rigidity of polynomials ensure that convergence of the rank occurs almost surely for a finite sample size.} 
				
\subsection{Assumptions and main result}
Recall that $\mu$ is a Borel probability measure on $\RR^p$ with all its moments finite and $V$ denote the Zariski closure of $\supp(\mu)$. 
An algebraic set is called irreducible if it is not the union of two algebraic proper subsets. Given an algebraic set $W$, we define a canonical area measure $\sigma_W$. Intuitively, an algebraic set consists of a finite union of a smooth submanifolds of dimension $p_W$ "glued" along a lower dimensional singular locus. The area measure is constructed with geometric integration theory techniques on the smooth part.

\paragraph{Construction of the area measure:} Let $W \subset \RR^p$ be an irreducible real algebraic set of dimension $p_W$. According to Proposition 3.3.14, Proposition 3.3.10 and Definition 3.3.4 in \cite{bochnak1998real}, there exists a lower dimensional algebraic subset $Y$ of dimension strictly smaller than $p_W$, such that the set $Z = W \setminus Y$ can be seen as a $p_W$ dimensional smooth submanifold of $\RR^p$ (possibly not connected). There is a natural Euclidean $p_W$ dimensional density on $Z$ induced from the euclidean embedding, see for example Section 7.3 in \cite{duistermaat2004multi}. This defines integration of continuous functions on $Z$; then Riesz representation theorem yields a regular positive Borel measure representing this very integration linear functional. The resulting measure is called the {\it area measure} of $W$ and it is denoted in short by $\sigma_W$.


\begin{assumption}
				$\mu$ is a Borel measure on $\RR^p$ with finite moments and $V$ is the Zariski closure of $\supp(\mu)$ endowed with the area measure  $\sigma_V$. They satisfy the following constraints:
				\begin{enumerate}
								\item $V$ is an irreducible algebraic set,
								\item $\mu$ is absolutely continuous with respect to $\sigma_V$.
				\end{enumerate}
				\label{ass:measReg}
\end{assumption}
Under this assumption one infers the following finite sample stabilization result. The statement formalizes the intuition that the set of polynomials on $\supp(\mu)$ and the set of polynomials on a finite sample are both isomorphic to the set of polynomials on $V$, almost surely with respect to the random draw of the sample, as long as the sample size reaches the dimension of the space $L_d^2(\mu)$.

\begin{theorem}
				For all $d \geq 1$ and all $n \geq \rank(\MM_{\mu,d})$, it holds almost surely that, 
				\begin{itemize}\itemsep.5em
								\item $\rank(\MM_{\mu_n,d}) = \rank(\MM_{\mu,d})$, 
								\item $\left\{ P \in \RR_d[X];\; \sum_{i=1}^n (P(\x_i))^2 = 0 \right\} = \mathcal{I}_d$.
				\end{itemize}
				\label{th:finiteSample}
\end{theorem}
A proof of Theorem \ref{th:finiteSample} is given in Section \ref{sec:finiteSampleApproxProof}.
It is the combination of  Theorem \ref{th:finiteSample} and Proposition \ref{prop:idealRank} which shows that, in principle, it is possible to fully characterize algebraic geometric properties of the Zariski closure of $\supp(\mu)$ from a finite sample, with probability one.

It is obvious that $\rank(\MM_{\mu_n,d}) \leq \rank(\MM_{\mu,d})$. The main idea of the proof is to show that rank deficiency implies that $\left\{ \x_i \right\}_{i=1}^n$ are in the zero locus of a polynomial $P$ on $V^n$ not vanishing everywhere on $V^n$. We then use the fact that a polynomial $P$ vanishing in a regular euclidean neighborhood in $V^n$ actually vanishe on the whole $V^n$. There are several known, and rather technical facts contributing to this statement. We simply mention them, referring the interested readert to some widely circulated texts.

First, dealing with an irreducible real algebraic set $W$ of dimension $k$ in euclidean space of dimension $m$, one defines its associated ideal $I(W)$ as the collection of all real polynomials vanishing on $W$. It is non-trivial to prove that the set $W$ possesses regular points, that is points $x \in W$ and elements $P_1,P_2, \ldots, P_{m-k} \in I(W)$ such that the rank of the jacobian matrix $\left[\frac{\partial P_i}{\partial x_j}]\right]$ is maximal, equal to $m-k$ in $B \cap W$, where $B$ is a euclidean ball centered at $x$, see Propositions 3.3.10 and 3.3.14 in \cite{bochnak1998real}.

Second, the euclidean ball $B \cap W$ is Zariski dense in the irreducible set $W$. While this fact is intuitive, its proof passes through complex variables. Specifically, one changes the base algebra and considers the ideal generated by $I(W)$ in the ring $A$ of complex analytic functions, in the same number of variables. There one invokes the familiar complex nullstellensatz for analytic functions and irreducibility of $W$ over the complex field, to infer that a real polynomial $h$ vanishing on $B \cap W$ belongs to the ideal $I(W)$. Finally flatness of completion of local rings allows to decide that $h \in I(W)$, see for instance \cite{serre1956geometrie}. 

We remark that the assumptions on $V$ and $\mu$ in \ref{ass:measReg} are mandatory for the validity of the above theorem. Two simple examples illustrate our claim.
\bigskip

\begin{example}
				Consider the following generative process in $\RR^2$: $x$, $y$ and $z$ are drawn independently from the uniform measure on $[-1,1]$. If $z \leq 0$, return $(0,y)$, otherwise return $(x, 0)$.

				The underlying measure $\mu$ is the ``uniform'' measure on the set $\left\{ (x,y) \in [-1,1]^2;\; xy=0 \right\}$ and hence we have $V = \left\{ (x,y) \in \RR^2;\; xy = 0 \right\}$. Draw an independent sample from $\mu$, $\left\{ (x_i,y_i) \right\}_{i=1}^n$, the event $y_i = 0$ for $i = 1,\ldots, n$ occurs with probability $1/2^n$. Hence there is a nonzero probability that our sample actually belong to the set $\tilde{V} = \left\{ (x,y) \in \RR^2;\; y=0 \right\}$. This will result in $\rank(\MM_{\mu_n,d}) < \rank(\MM_{\mu,d})$. Since this event holds with non zero probability, this shows that Theorem \ref{th:finiteSample} may not hold for any value of $n$ if $V$ is not irreducilble.
\end{example}

\begin{example}
				Similarly, absolute continuity is necessary. Consider for example a mixture between an absolutely continuous measure and a Dirac measure. There is a non zero probability that the sample only contains a singleton which will induce a rank deficiency in the moment matrix.
\end{example}
\section{Reference measure with uniform asymptotic behavior}
\label{sec:sphere}
This section describes an asymptotic relation between the Christoffel function and the density associated with the underlying measure. The proof is mostly adapted from \cite[Theorem 1.1]{kroo2012christoffel}, we provide explicit details and use a simplified set of assumptions which allows us to deal with singular measures.

\subsection{Main assumptions and examples}
Our main hypothesis is related to a special property of a reference measure whose existence is assumed. Assumption \ref{ass:densityEstimation} describes the fact that the Christoffel function of a reference measure $\lambda$ is asymptotically of the same order for all $z$ in $Z$. This property will in turn be used to establish a connection between the asumptotic of the Christoffel function of a measure $\mu$, such that $\mu$ and $\lambda$ are mutually absolutely continous, and the density $d\mu/d\lambda$ in Theorem \ref{th:asymptoticsUnitSphere}.
\begin{assumption}
				Let $Z$ be a compact subset of $\RR^p$ and assume that there exists a reference Borel probability measure $\umeas$ whose support is $Z$ and a polynomial function $N \colon \RR_+\mapsto \RR_+^*$ such that
				\begin{align*}
								\lim_{d \to \infty} \sup_{z \in Z} |N(d) \Lambda_{\umeas,d}(z) - 1| = 0.
				\end{align*}
				\label{ass:densityEstimation}
\end{assumption}
\begin{remark}
				The constant $1$ is arbitrary and could be replaced by a continuous and strictly positive function of $z$.
				\label{rem:arbitraryFun}
\end{remark}
Assumption \ref{ass:densityEstimation} requires that the support of $\lambda$ does not have a boundary in $Z$. Recall that $\int_{z \in Z} \kappa_{\lambda,d}(z,z) d\lambda(z) = \rank(\MM_{\lambda,d})$ is the dimension of the space of polynomials on $Z$ (modulo equality on $Z$) and hence $N(d)$ has to be proportional to this dimension. 

The principal example that fits Assumption \ref{ass:densityEstimation} is the area measure on the $p-1$ dimensional sphere in $\RR^p$, denoted $\mathbb{S}^{p-1}$, which is isotropic and hence which Christoffel function is constant.
The dimension $N(d)$ of the vector space of polynomials over $\mathbb{S}^{p-1}$ is given by:
\begin{align*}
			{p+d-1 \choose p-1} + {p+d-2 \choose p-1} = \left( 1 + \frac{2d}{p-1} \right) {d+p - 2 \choose p - 2}
\end{align*}
which grows with $d$ like $\frac{2d^{p-1}}{p-1}$ for a fixed value of $p$.
As a function of $d$ it is exactly the Hilbert polynomial associated with the real algebraic set $\mathbb{S}^{p-1}$. Here we may choose for $\umeas$ the rotation invariant probability measure 
on $\mathbb{S}^{p-1}$ (as normalized area measure). In this case $\Lambda_{\umeas,d}(\z)$ is proportional to $\frac{1}{N(d)}$ for all $\z \in \mathbb{S}^{p-1}$ and all $d$ by rotational invariance of both the sphere and the Christoffel function.\\

The case of the sphere is important because it helps 
construct many more situations which satisfy Assumption \ref{ass:densityEstimation}:
			\begin{itemize}\itemsep-.1em
							\item Product of spheres with products of area measures: the bi-torus in $\RR^4$ with the corresponding area measure.
							\item Affine transformations of such sets with the push forward of the reference measure with respect to the affine map; for instance the ellipsoid,
										\end{itemize}
or rational embeddings of the sphere in higher dimensional space, to mention only a few natural choices of admissible operations.

\subsection{Main result}
Given a reference measure $\umeas$ as in Assumption \ref{ass:densityEstimation}, and another measure $\mu \sim \umeas$ (i.e. $\umeas\ll\mu$ and $\mu\ll \umeas$), one can describe a precise relation with the underlying density.
\begin{theorem}
				Let $Z$ and $\umeas$ satisfy Assumption \ref{ass:densityEstimation}. Let $\mu$ be a Borel probability measure on $\RR^p$, absolutely continuous with respect to $\umeas$, with density $f\colon Z \to \RR_+^*$  which is continuous and positive.
				Then:
				$$\lim_{d \to \infty}\sup_{\x \in Z} |N(d)\Lambda_{\mu,d}(\x) - f(\x)| = 0.$$
				\label{th:asymptoticsUnitSphere}
\end{theorem}
A proof of Theorem \ref{th:asymptoticsUnitSphere} is given in Section \ref{sec:sphere}.
In the case of the sphere, Assumption \ref{ass:densityEstimation} does not hold only in the limit and 
from the proof of Theorem \ref{th:asymptoticsUnitSphere} we are able
to obtain  explicit quantitative bounds and a convergence rate which could be arbitrarily close to $O(d^{-1/2})$. To the best of our knowledge, this is the first estimate of this type for the Christoffel function. 

\begin{corollary}
				Let $f \colon \RR^p \to \RR_+$ be Lipschitz on the unit ball with $0 < c \leq f \leq C < +\infty$ on $\mathbb{S}^{p-1}$ and assume that $\mu$ has density $f$ with respect to the uniform measure on the sphere. Then for anu $\alpha \in (1/2,1)$, $$\sup_{\x \in S^{p-1}} |N(d) \Lambda_{\mu,d}(\x) - f(\x)| = O(d^{\alpha - 1}).$$ 
				\label{lem:explicitSphere}
\end{corollary}

\section{Numerical experiments}
\label{sec:numerics}
\begin{figure}[t]
				\centering
				\includegraphics[width = \textwidth]{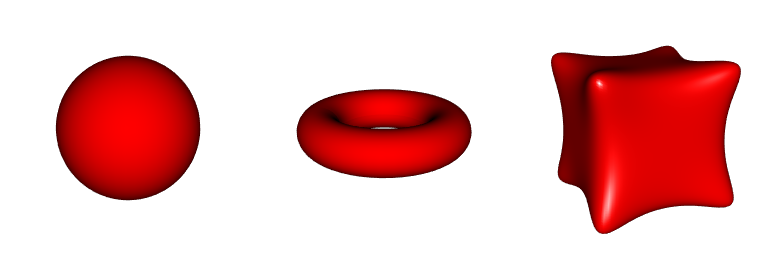}
				\includegraphics[width = \textwidth]{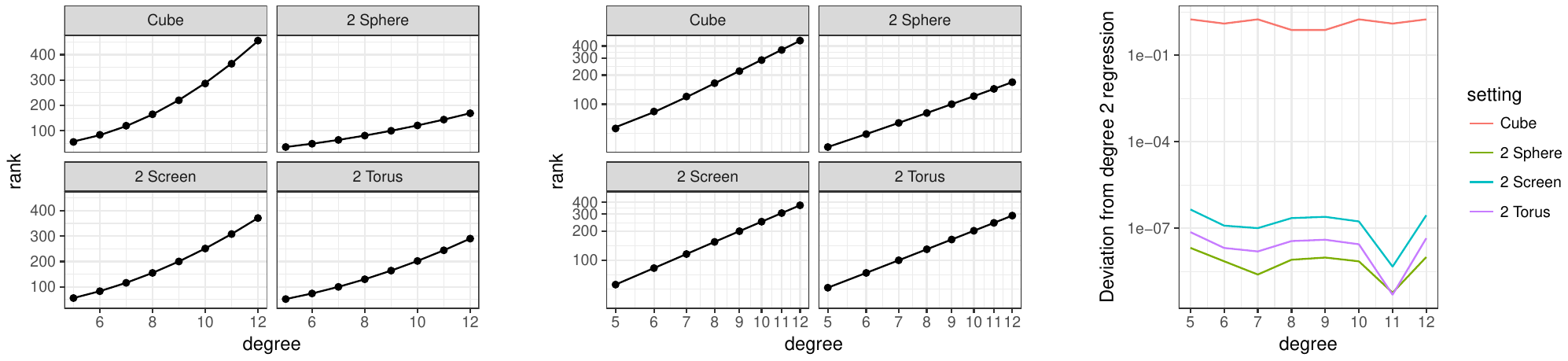}
				\caption{Top: A visualization of the 2 dimensional surfaces considered in this example, the sphere, the torus and the TV screen. Bottom: Relation between the rank of the moment matrix and the corresponding degree bound. For different sets, the dots represent the measured rank and the curve is the degree-two interpolation. On the left, the raw relation, we see that the cube has the highest rank. The same plot is in log log scale in the middle. The difference between measured rank and estimated degree-two interpolation is hardly visible. On the right, we represent the residuals between degree-two interpolation and measured ranks. The degree-rank relation is well interpolated for two dimensional sets while this is not the case for the cube.}
				\label{fig:rank}
\end{figure}
\subsection{Rank of the moment matrix}
We illustrate Proposition \ref{prop:idealRank} by the following numerical experiment:
\begin{itemize}\itemsep-.1em
				\item Sample $20000$ points on a chosen set $\Omega \subset \RR^3$, from a density with respect to the area measure on $\Omega$.
				\item For  $d = 5, \ldots, 12$, compute the rank of the empirical moment matrix: set $X \in \RR^{20000 \times s(d)}$ the design matrix representing each point in a polynomial basis (of size $s(d)$), such that the empirical moment matrix is given by $1/20000 X^TX$. We estimate the rank of $X$ using singular value thresholding (multivariate Chebyshev basis, threshold $10^{-10}$).
				\item Fit a degree-$2$ regression polynomial interpolating the relation between the degree and the rank.
\end{itemize}
We choose four different subsets of $\RR^3$: unit cube, unit sphere, TV screen, torus.
The first one is $3$-dimensional while all the others are 
$2$-dimensional (see Figure \ref{fig:rank}). From Proposition \ref{prop:idealRank}, 
in the first case it is expected that the computed rank grows like a third degree polynomial while for the remaining cases, it should grow like a quadratic. Hence the interpolation of the rank-degree relation should be of good quality for the last three cases and not for the first case. This is what we observed, see Figure \ref{fig:rank}.
\subsection{Density estimation on an algebraic set}
\label{sec:experiments}
We present multivariate datasets whose topological characteristics suggest to map them to algebraic sets capturing symmetries. Practitioners have developed density estimation tools which mostly rely on the ability to compute a distance like divergence between two points which respects the topology of the data. As we next show, a {\em unifying and generic approach} allows to treat all these cases using the same computational tool: the empirical Christoffel function. 

The first step consists in mapping the data of interest on an algebraic set whose topology reflects the intrinsic topology of the data, namely: the circle for periodic data, the sphere for celestial data, and the torus for bi-periodic data. Then, we evaluate the empirical Christoffel function on the chosen set and use it as a proxy to density. We use the pseudo-inverse of the empirical moment matrix and evaluate the Christoffel function on a grid through formulas \eqref{eq:christoffelDarboux} and \eqref{eq:fundamentalIdentity}. We then plot the contours of the estimate obtained on the grid to get a graphical representation of the estimated density.

The Christoffel function highly depends on the geometry of the boundary of the support. The algebraic sets considered here {\em do not} have boundaries (as manifolds) and isotropy properties ensure that the Christoffel function associated to the {\em uniform} measure on these sets is {\em constant}.

\subsection{Dragonfly orientation: the circle}
The dataset was described in \cite{batschelet1981circular} and consists of measurements of the orientation of 214 dragonflys with respect to the azimuth of the sun. The orientation is an angle which has a periodicity and as such is naturally mapped to the circle. The dataset and the corresponding Christoffel function are displayed in Figure \ref{fig:dragonflyStar}. As the degree increases, the Christoffel function captures regions densely populated by observations (high density) and regions without any observation (low density). As was already observed in \cite{batschelet1981circular}, dragonflies tend to sit in a direction perpendicular to the sun.
\subsection{Double stars: the sphere}
\begin{figure}
				\centering
				\includegraphics[width=.3\textwidth]{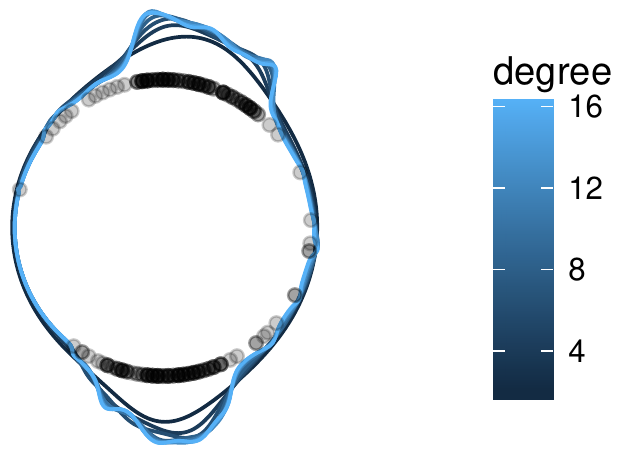}\qquad
				\includegraphics[width=.4\textwidth]{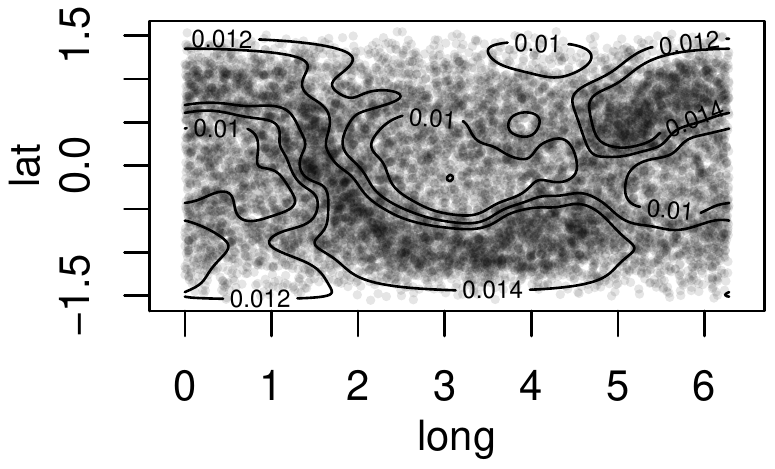}
				\caption{Left: Dragon fly orientation with respect to the sun, on the torus. The curves represent the empirical Christoffel function and the dots are observations.
				Right: Each point represents the observation of a double star on the celestial sphere, associated to longitude and latitude. The level sets represent the empirical Christoffel functions on the sphere in $\RR^3$ (degree $8$). The highlighted band corresponds to the Milky Way.}\label{fig:dragonflyStar}
\end{figure}
We reproduce the experiment performed in \cite{cuevas2006plugin}. The dataset is provided by the European Space agency and was aquired by Hipparcos satelite \cite{european1997hipparcos}. The data consists of the position of 12176 double stars on the celestial sphere described by spherical coordinates. Double stars are of interest in astronomy because of their connection with the formation of evolution of single stars. A natural question is that of the uniformity of the distribution of these double stars on the celestial map. The dataset and corresponding Christoffel function are displayed in Figure \ref{fig:dragonflyStar} using equirectangular projection. Firstly we note that the displayed level lines nicely capture the geometry of the sphere without distortion at the poles. Secondly the Christoffel function allows to detect a higher density region (above 0.14) which corresponds to the Milky Way.

\subsection{Amino-acid dihedral angles: the bi-torus}
We reproduce the manipulations performed in \cite{lovell2003structure}. Proteins are amino acid chains which 3D structure can be described by $\phi$ and $\psi$ backbone dihedral angle of amino acid residues. The 3D structure of a protein is extremely relevant as it relates to the molecular and biological function of a protein. Ramachandran plots consist of a scatter plot of these angles for different amino acids and allow to visualize energetically allowed configuration for each amino acid. 

It is worth emphasizing that being able to describe typical regions in Ramachandran plots is of great relevance as a tool for {\em protein structure validation \cite{lovell2003structure}}. Since the data consist of angles, it has a bi-periodic structure and therefore naturally maps to the bi-Torus in $\RR^4$. A Ramachandran plot for 7705 Glycine amino acids as well as the corresponding Christoffel function estimate is displayed in Figure \ref{fig:dihedral}. The Christoffel function is able to identify highly populated areas (density above 0.08) and its level set nicely fit the specific geometry of the torus. We refer the reader to \cite{lovell2003structure} for more details about this dataset.

\begin{figure}
				\centering
				\includegraphics[width=.6\textwidth]{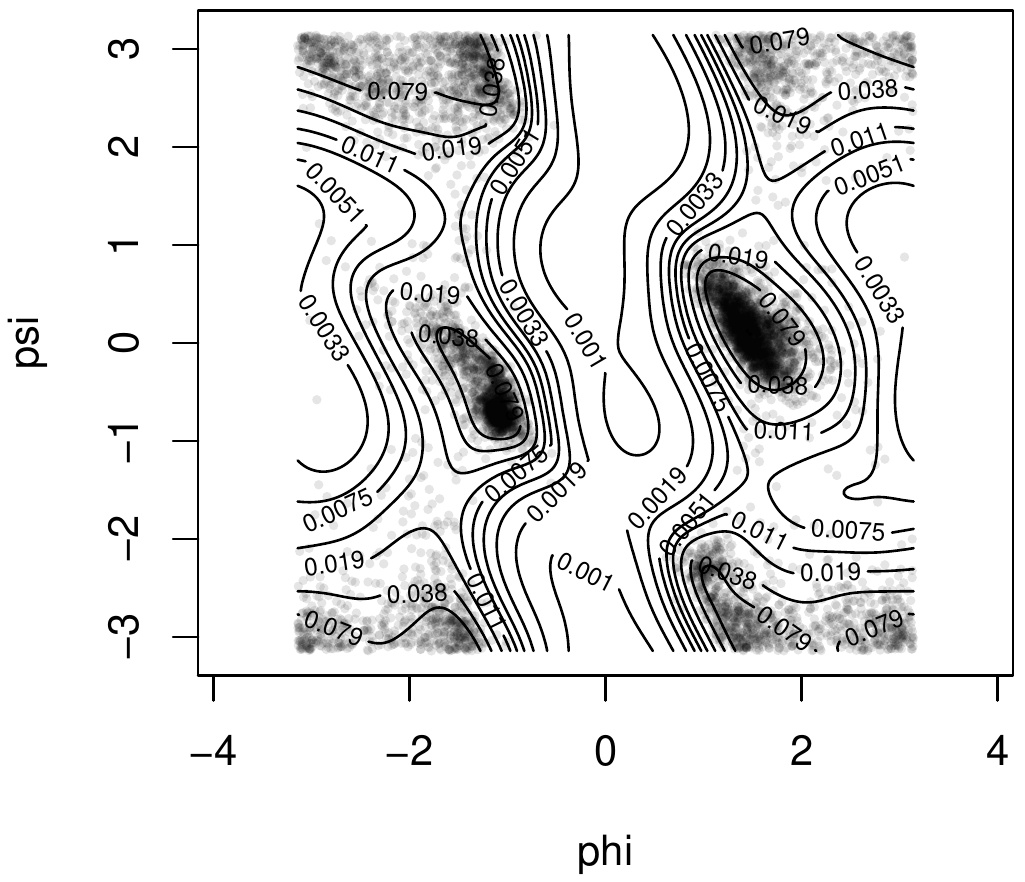}
				\caption{Each point represents two dihedral angles for a Glycine amino acid. These angles are used to describe the global three dimensional shape of a protein. They live on the bitorus. The level sets are those of the empirical Christoffel functions evaluated on the sphere in $\RR^4$. The degree is $4$.}
				\label{fig:dihedral}
\end{figure}

\section{Discussion: perturbation of the moment matrix}
\label{sec:discussion}
We import real algebraic geometry and approximation theory results as tools to infer qualitative, topological properties from datasets. This is intrinsically connected to the notion of algebraic sets and allows to leverage more rigidity than provided by usual smoothness assumptions. This is suited for data restricted to algebraic sets.

However, in a singular situation, we are faced with a numerical instability issue. Aside from exact real arithmetic computation, the notion of rank in finite precision arithmetic is actually ill-defined which makes the singularity assumption questionable. We believe that the objects considered in this paper have sufficient stability properties to treat ``near-singular'' cases.  We describe an heuristic argument, while more refined quantitative results remain to be investigated in a future work.
\begin{lemma}
				For a given $m \in \NN$, let  $\mathbf{A},\MM \in \RR^{m \times m}$ be symmetric semidefinite. Then, for any $\v \in \RR^m$ (uniformly on any compact)
				\begin{align*}
								&\lim_{l \to 0_+} \inf_{\x \in \RR^m} \left\{ \x^T (\MM + l \mathbf{A}) \x \quad {\rm s.t.}\, \x \in \RR^p,\,  \x^T\v = 1    \right\} \\
								=\quad & \inf \left\{\x^T \MM  \x \quad {\rm s.t.}\, \x \in \RR^m,\,  \x^T\v = 1    \right\}.
				\end{align*}
				\label{lem:perturb}
\end{lemma}
By Lemma \ref{lem:perturb}, the Christoffel function associated with 
a slight perturbation of the moment matrix is very close to the actual Christoffel function. This justifies the variational formulation \eqref{def-christo-d} as it can be seen as the limit of perturbations of $\mu$ making it non singular. By continuity of eigenvalues an appropriate thresholding scheme should lead to a correct evaluation of the rank of the moment matrix.

An important application of Lemma \ref{lem:perturb} is the addition of noise. Consider the following random process $\y = \x + \epsilon$,
where $\x$ is distributed according to $\mu$ and $\epsilon$ is independent small noise. Measuring the impact of $\epsilon$ on the moments of $\y$ compared to moments of $\x$ will help using our tools in the ``close to singular'' case. Understanding of the singular situation is a key to investigate robust variants suited to more general and practical \emph{manifold learning} situations.

An important example of application of Lemma \ref{lem:perturb} is the case of a convolution, which corresponds to the addition of noise. Consider the following random process
For illustration purpose,
in Figure \ref{fig:convolution} one observes how the level sets tend to concentrate
as the perturbation gets smaller.
\begin{figure}
				\centering
				\includegraphics[width=.6\textwidth]{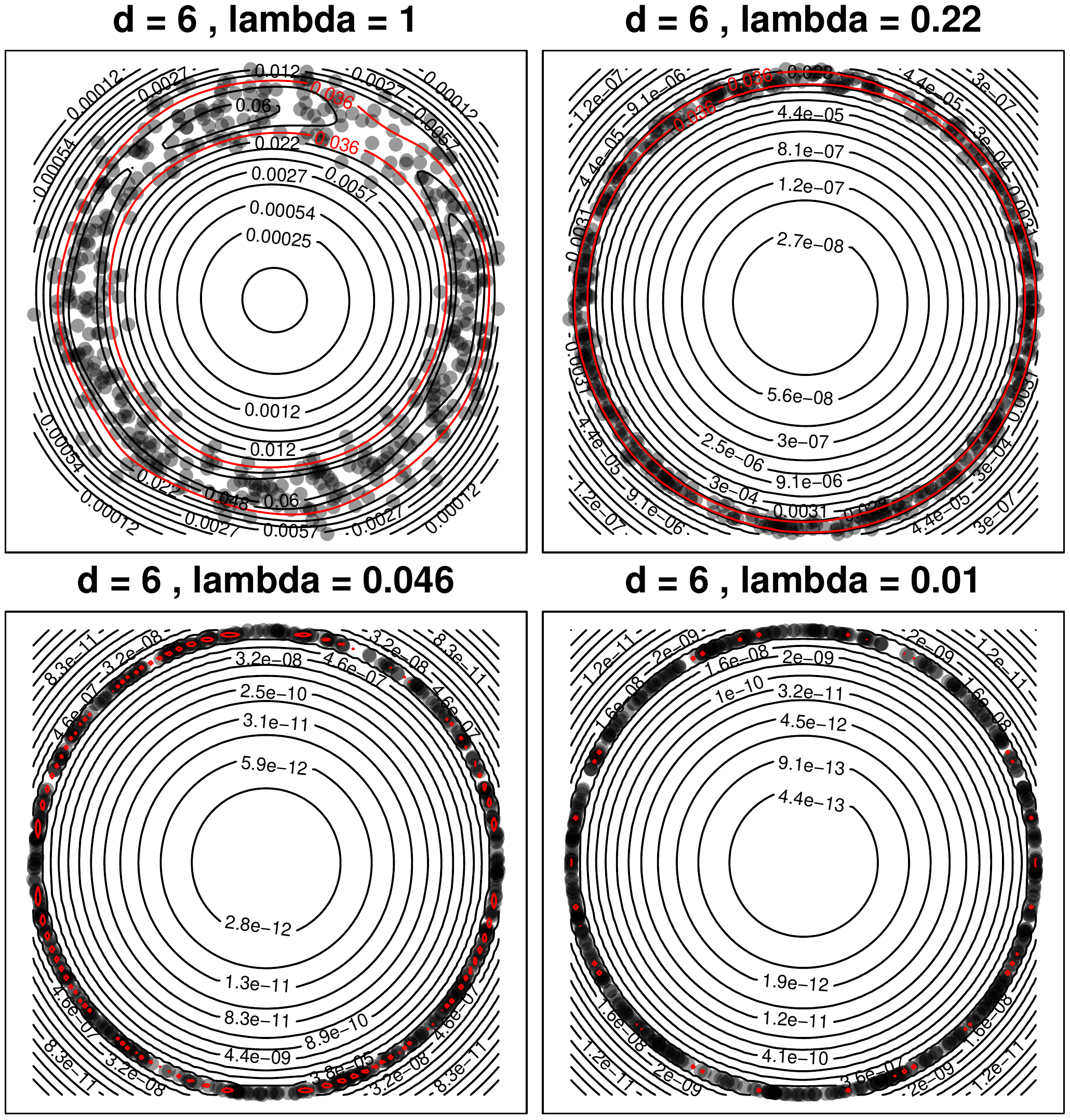}
				\caption{Small perturbation of a set of points on the circle by addition of a small amount of noise. The perturbation tends to 0 and the level sets of the Christoffel function tend to converge to the circle as predicted by Lemma \ref{lem:perturb}.}
				\label{fig:convolution}
\end{figure}

\section{Proofs}
\label{sec:proofs}
This section contains proofs of the main results of the present paper.

\subsection{Rank of the moment matrix}
We begin with the proof of Lemma \ref{lem:radical}.
\begin{proof}[Proof of Lemma \ref{lem:radical}:]
First, $P \in \mathcal{I}$ entails that $P$ vanishes on $\supp(\mu)$ by continuity of $P$. If $P$ did not vanish on $V$, then one would construct a strictly smaller algebraic subset of $V$ containing $\supp(\mu)$ which is contradictory. Conversely, let $P\in\R_d[X]$  vanish on $V$, then $P$ vanishes on $\supp(\mu)$ and $P(\x)^2=0$ for all $\x\in\supp(\mu)$ and therefore $\int P^2d\mu=0$, which in turn implies $P\in\mathcal{I}$.
\end{proof}

We introduce the following notion of genericity, borrowed from \cite{laurent2012approach}. For any $d \in \NN$, let $\mathcal{I}_d:=\mathcal{I}\cap\R_d[x]$ be the intersection of the ideal $\mathcal{I}$ with the vector space of polynomials of degree at most $d$.
\begin{defn}
				For any fixed $d \in \NN$, let $s(d) = {p+d \choose d}$, denote by $\mathcal{K}_d$ the set
				\begin{align*}
								\mathcal{K}_d = \left\{ \MM \in \RR^{s(d) \times s(d)},\, \MM_{11} = 1,\, \MM \succeq 0,\, \p^T\MM\p = 0,\, \forall \p \in \mathcal{I}_d\right\}.
				\end{align*}
				A matrix $\mathbf{T} \in \RR^{s(t) \times s(t)}$ is called generic if $\mathbf{T} \in \mathcal{K}_d$ and $\rank(\mathbf{T}) \geq \rank(\MM)$ for all $\MM \in \mathcal{K}_d$.
				\label{def:genericity}
\end{defn}

\begin{lemma}
				For any	$d \in \NN^*$, $\MM_{\mu,d}$ is generic.
				\label{lem:genricity}
\end{lemma}
\begin{proof}
				First, for any polynomial $P$ of degree at most $d$, with coefficient vector $\p$, we have $\p^T \MM_{\mu,d} \p = \int P^2 d\mu$. This quantity vanishes for $P \in \mathcal{I}_d$ and hence $\MM_{\mu,d} \in \mathcal{K}_d$. Choose $\MM \in \mathcal{K}_{d}$ and $\p \in \RR^{s(d)}$ such that $\MM_{\mu,d}\p = 0$. We have that $\p^T \MM_{\mu,d} \p = \int P^2 d\mu = 0$ so that  $P \in \mathcal{I}_d$ and hence $\MM\p = 0$. We have shown that $\ker(\MM_{\mu,d}) \subset \ker(\MM)$ so that $\rank(\MM_{\mu,d}) \geq \rank(\MM)$.
\end{proof}
We introduce the Hilbert function and the Hilbert polynomial. We refer the reader to \cite[Chapter 9]{cox2007ideals} for a general presentation. 
\begin{defn}
				The Hilbert Function of $\mathcal{I}$ is defined by
				\begin{align*}
								\HF \colon \NN &\mapsto \NN \\
								t &\mapsto \dim \RR_t[x] / \mathcal{I}_t. 
				\end{align*}
				\label{def:hilbertPoly}
\end{defn}
The following lemma defines the Hilbert polynomial.
\begin{lemma}
				For large values of $t$, $\HF$ is a polynomial. 
				\label{lem:hilbertPoly}
\end{lemma}
Intuitively, the Hilbert function associates to $d$ the dimension of the space of polynomials of degree up to $d$ on $\supp(\mu)$. This space is constructed using a natural quotient map as described in Section \ref{sec:singular}. 
Proposition \ref{prop:idealRank} combines \cite[Proposition 2]{laurent2012approach} and relation between Hilbert's polynomial and dimension \cite[Chapter 9]{cox2007ideals}. 

\subsection{Finite sample approximation, proof of Theorem \ref{th:finiteSample}}
\label{sec:finiteSampleApproxProof}
Recall that for any $d\in \NN^*$, $L_d^2(\mu)$ is identified with $\RR_d[x] / \mathcal{I}_d$, two polynomials are considered equivalent if they agree on $V$. $(L_d^2(\mu), \langle\cdot,\cdot\rangle_\mu)$ is a finite dimensional RKHS which reproducing kernel is symmetric positive definite and defined for all $\x \in V$, by $\kappa_{\mu,d}(\cdot,\x) \in L_d^2(\mu)$ such that for all $P \in L_d^2(\mu)$,
\begin{align*}
				\int P(\y) \kappa_{\mu,d}(\y,\x) d\mu(\y) = P(\x).
\end{align*}
For any $\x \in V$, $\kappa_{\mu,d}(\cdot, \x)$ is a polynomial and, by symmetry, $\kappa_{\mu,d}$ is itself a polynomial on $V^2$ which satisfies $\kappa_{\mu,d}(\x,\x) > 0$ for all $\x \in V$. Furthermore, we have $L_d^2(\mu) = \mathrm{span}\left\{ \kappa_{\mu,d}(\cdot,\x), \,\x \in \supp(\mu) \right\}$.
This space being finite dimensional, there exists $N(d) \leq s(d)$ such that one can find a basis $\e_1,\ldots, \e_{N(d)} \in \supp(\mu)$ with 
\begin{align}
				L_d^2(\mu) = \mathrm{span}\left\{ \kappa_{\mu,d}(\cdot,\e_i), \,i = 1,\ldots, N(d) \right\}.
				\label{eq:basis}
\end{align}
Furthermore, let $\left\{ P_i \right\}_{i=1}^{N(d)}$ be an orthonormal basis of $L_d^2(\mu)$. Then the identity $\int \kappa_{\mu,d}(x,y) P_i(y) d\mu(y) = P_i(x)$,
for all $\x \in V$ and all $i$, yields:
\begin{align}
				\kappa_{\mu,d}(\x,\y) = \sum_{i=1}^{N(d)} P_i(\x) P_i(\y), \qquad \forall\,\x,\y \in V.
				\label{eq:christoDarbouxKernel}
\end{align}
We begin by a technical Lemma which is a consequence of the reproducing kernel construction and its relation to orthogonal polynomials.
\begin{lemma}
				Let $\v_1,\ldots, \v_n\,\in V$ and let  $\mu_n$ be the empirical average 
				$\sum_{i=1}^n \delta_{\v_i}$. Then:
				\begin{align*}
								\rank\left( \left( \kappa_{\mu,d}(\v_i,\v_j) \right)_{i,j = 1\ldots n} \right) = \rank\left( \MM_{\mu_n,d} \right).
				\end{align*}
				\label{lem:rankEquality}
\end{lemma}
\begin{proof}
				This is a consequence of equation (\ref{eq:christoDarbouxKernel}). Let $\left\{ P_i \right\}_{i=1}^{N(d)}$ be an orthonormal basis of $L_d^2(\mu)$. For each $i = 1 ,\ldots, N(d)$ we choose $Q_i \in \RR_d[X]$ to be one element in the equivalence class of $P_i$. The family $\left\{Q_i\right\}_{i=1}^{N(d)}$ must be independent in $\RR_d[X]$ otherwise this would contradict independence of $\left\{ P_i \right\}_{i=1}^{N(d)}$ in $L_d^2(\mu)$. This independent family can be extended to a basis of $\RR_d[X]$ which we denote by $\left\{Q_i\right\}_{i=1}^{s(d)}$. For each $j = N(d) + 1, \ldots, s(d)$ we may enforce that $Q_j$ vanishes $V$ by substracting if necessary a linear combination of $\left\{Q_i\right\}_{i=1}^{N(d)}$ agreeing with $Q_j$ on $V$. Consider the matrix $D = \left(Q_j(\v_i)  \right)_{i = 1\ldots n,\,j = 1\ldots s(d)}$. Then $\rank\left( \left( K(\v_i,\v_j) \right)_{i,j = 1\ldots n} \right) = \rank(D D^T)$ from \eqref{eq:christoDarbouxKernel} and $ \rank\left( \MM_{\mu_n,d} \right) = \rank(D^T D)$ since the rank of the moment matrix does not depend on the choice of polynomial basis in $\RR_d[X]$. Both ranks are the same which is the desired result.  
\end{proof}
Now let $V^{N(d)}$ be the cartesian product space 
$\underbrace{V\times\cdots \times V}_{\mbox{$N(d)$ times}}$. It is irreducible \cite[Theorem 2.8.3]{bochnak1998real} and its area measure
is the product measure $\sigma = \otimes_{i=1}^{N(d)} \sigma_V$. The determinantal function:
\begin{align*}
				F \colon V^{N(d)} &\mapsto \RR\\
				(\x_1, \ldots, \x_n) &\mapsto \det\left( \left( K(\x_i,\x_j) \right)_{i,j = 1, \ldots, N(d)} \right)
\end{align*}
is a polynomial function which is not identically zero on $V^{N(d)}$, since by (\ref{eq:basis}), $F(\e_1, \ldots, \e_{N(d)}) > 0$. The product structure of $\sigma$ reflects independence of the random sample with common distribution $\sigma_V$ in Assumption \ref{ass:measReg}. The polynomial $F$ is to be integrated with respect to $\sigma$. Using the observation that $F$ is nonzero implies that the empirical moment matrix has full rank, we will actually show that $F$ is non zero $\sigma$ almost everywhere and deduce absence of rank deficiency with probability one.
The following observation, of some independent interest, is in order.
\begin{lemma}
				Let $W$ be an irreducible real algebraic set with corresponding area measure 
				$\sigma_W$ and $P$ be a polynomial on $W$. The following are equivalent
				\begin{itemize}
								\item[(i)] $\sigma_W\left( \left\{ \z \in W:\: P(z) = 0 \right\} \right) > 0$.
								\item[(ii)] $P$ vanishes on $W$.
				\end{itemize}
				\label{lem:irreducibleVanish}
\end{lemma}
This lemma follows from the fact that a euclidean neighborhood of a regular point of $W$ is Zariski dense in $W$, assuming irreducibility. See also the comments after the statement of Theorem \ref{th:finiteSample}. We describe a detailed proof in Appendix \ref{sec:lemmas} for self containedness.

We may now proceed to the proof of the rank stabilization result from Theorem \ref{th:finiteSample}. First, since $V$ is irreducible, $V^{N(d)}$ is also irreducible \cite[Theorem 2.8.3]{bochnak1998real}. Since $F$ is a polynomial which does not vanish everywhere on $V^{N(d)}$, we deduce from Lemma \ref{lem:irreducibleVanish} that 
\begin{align}
				&\sigma\left( \left\{ \x_1,\ldots, \x_{N(d)} \in V^{N(d)}:\: F(\x_1, \ldots, \x_{N(d)}) = 0\right\} \right)\nonumber \\
				= &\int_{V^{N(d)}} \mathbb{I}\left[ F(\x_1, \ldots, \x_{N(d)}) = 0 \right] d \sigma(\x_1,\ldots, \x_{N(d)}) \nonumber\\
				= &\ 0.
				\label{eq:vanishAreaMeas}
\end{align}
Noticing that $\rank(\MM_{\mu,d}) = N(d)$, the result follows because $\mu$ is absolutely continuous with respect to $\sigma_V$ and hence $\otimes_{i=1}^{N(d)} \mu$ is absolutely continuous with respect to $\sigma$ so that, combining Lemma \ref{lem:rankEquality} and (\ref{eq:vanishAreaMeas}) with the i.i.d. assumption, yields:
\begin{align*}
				\PP\left(\rank\left( \MM_{\mu_n,d} < N(d)\right)\right) &= \PP\left(F(x_1,\ldots, x_{N(d)}) = 0\right) \\
				&= \int_{V^{N(d)}} \mathbb{I}\left[ F(x_1, \ldots, x_{N(d)}) = 0 \right] d_\mu(x_1) \ldots d_\mu(x_{N(d)}) = 0.
\end{align*} 
This proves the first point. The second point is a consequence of the first point since the existence of a polynomial vanishing on the sample but not on the whole $V$ would induce a rank deficiency in the empirical moment matrix and this occurs with zero probability.

\subsection{Reference measure with uniform asymptotic behavior, proof of Theorem \ref{th:asymptoticsUnitSphere}}
\label{sec:sphereProof}
Although the proof is mostly adapted from \cite[Theorem 1.1]{kroo2012christoffel}, we provide explicit details and use a simplified set of assumptions which allows us to deal with singular measures.
\begin{proof}[Proof of Theorem \ref{th:asymptoticsUnitSphere}:]
				Our arguments, mostly inspired from \cite{kroo2012christoffel}, are
				adapted to our setting. In addition, we provide some novel quantitative details.
				We split the proof into two parts. The second part is essentially a repetition of arguments similar to those used in the first part.
				\paragraph{One direction}
				By Assumption \ref{ass:densityEstimation},  $Z$ is compact and so denote by $D$ the finite diameter of $Z$. Fix an arbitrary $d \in \NN^*$ and set:
				\begin{align}
								\epsilon_d = \sup_{\x,\y \in Z,\,\|\x- \y\| \leq D / \sqrt{d}} |f(\x) - f(\y)|,
								\label{eq:asymp1}
				\end{align}
				which is well-defined because $f$ is continuous on the compact set $Z$ and hence uniformly continuous on $Z$, which ensures that the supremum is finite \footnote{Remark \ref{rem:explicitSphere} provides a justification of the choice of a decrease of order $1/\sqrt{d}$, this achieves balance between two terms in the case of Lipschitz densities}. Furthermore, $\epsilon_d \to 0$ as $d \to \infty$.

				Next, fix an arbitrary $\x_0 \in Z$ and $\alpha \in (1/2,1)$. We will prove a convergence rate of order $d^{\alpha - 1}$, the flexibility in the choice of $\alpha$ allows to get arbitrarily close to $d^{-1/2}$. Choose $P_{d - \lfloor d^{\alpha}\rfloor}$ realizing the infimum for $\Lambda_{\umeas,d - \lfloor d^{\alpha}\rfloor}(\x_0)$ ($\lfloor\cdot\rfloor$ denotes the floor integer part). Set $Q_{2\lfloor d^{\alpha}/2\rfloor} \colon \y \mapsto Q((\y-\x_0)/D)$ where $Q$ is given by Lemma \ref{lem:needlePoly} with $\delta = 1/\sqrt{d}$ and of degree at most $2\lfloor d^{\alpha}/2\rfloor$. Then:
				\begin{align*}
								2\lfloor d^{\alpha}/2\rfloor \leq \lfloor d^{\alpha}\rfloor,
				\end{align*}
				and hence, the polynomial $P = P_{d - \lfloor d^{\alpha}\rfloor}Q_{2\lfloor d^{\alpha}/2\rfloor}$ is of degree at most $d$ and satisfies 
				\begin{align}
								\begin{cases}
									P(\x_0) = 1\\
									|P| \leq |P_{d - \lfloor d^{\alpha}\rfloor}| \text{ on } Z\\
									P^2 \leq 2^{2 - d^{\alpha - 1/2}/2} \sup_{\x \in Z} P_{d - \lfloor d^{\alpha}\rfloor}(x)^2 \text{ on } Z \setminus B_{D/\sqrt{d}}(\x_0),	
								\end{cases}
								\label{eq:asymp2}
				\end{align}
				where the first identity is because $P$ is a product of polynomial 
				whose value at $\x_0$ is $1$, the last two identities follow by maximizing both terms of the product, using Lemma \ref{lem:needlePoly} and the fact that $1 - \delta \lfloor d^{\alpha}/2\rfloor \leq 2 - \delta d^{\alpha} / 2  =  2 - d^{\alpha - 1 / 2} / 2$. 
				$P$ is feasible and provide an upper bound for the computation of the Christoffel function of $\mu$. We obtain:
		\begin{align*}
						&\Lambda_{\mu,d}(\x_0) \\
						\overset{(i)}{\leq} \quad & \int_{Z} (P(\x))^2 d\mu(\x)\\
						\overset{(ii)}{\leq} \quad & \int_{Z \cap \B(\x_0, D/\sqrt{d})} (P_{d - \lfloor d^{\alpha}\rfloor}(\x))^2f d\umeas(\x) + 2^{2 - d^{\alpha - 1/2}/2} \sup_{\x \in Z} P_{d - \lfloor d^{\alpha}\rfloor}(\x)^2 \int_{Z \setminus \B(\x_0, D/\sqrt{d})} d\mu(\x)\\
						\overset{(iii)}{\leq} \quad & (f(\x_0 ) + \epsilon_d) \int_{Z \cap \B(\x_0, D/\sqrt{d})} (P_{d - \lfloor d^{\alpha}\rfloor}(\x))^2 d\umeas(\x) \\
						&+ 2^{2 - d^{\alpha-1/2}/2} \sup_{\z \in Z} \left(\Lambda_{\umeas,d - \lfloor d^{\alpha}(\z)\rfloor}  \right)^{-1} \int_{Z} (P_{d- \lfloor d^{\alpha}\rfloor}(\x))^2 d\umeas(\x)\\
						\overset{(iv)}{\leq} \quad & \left(f(\x_0 ) + \epsilon_d + 2^{2 - d^{\alpha - 1/2}/2}\sup_{\z \in Z}\left( \Lambda_{\umeas,d - \lfloor d^{\alpha}\rfloor}(\z)  \right)^{-1}\right) \int_{Z} (P_{d - \lfloor d^{\alpha}\rfloor}(\x))^2 d\umeas(\x) \\
						= \quad & \left(f(\x_0 ) + \epsilon_d + 2^{2 - d^{\alpha-1/2}/2}\sup_{\z \in Z}\left( \Lambda_{\umeas,d - \lfloor d^{\alpha}\rfloor}(\z)  \right)^{-1}\right)  \Lambda_{\umeas,d - \lfloor d^{\alpha}\rfloor}(\x_0).
		\end{align*}
		where $(i)$ is because $P(\x_0) = 1$, $(ii)$ follows by decomposition of the integral over two domains and the uniform bounds in (\ref{eq:asymp2}), $(iii)$ follows by combining (\ref{eq:asymp1}), Lemma \ref{lem:firstProperties} and the fact that $\mu$ is a probability measure on $Z$, $(iv)$ follows by extending the domain of the first integral and the last identity is due to the choice of $P_{d - \left\lfloor d^{\alpha}\right \rfloor }$ and Lemma \ref{lem:firstProperties}. 
Therefore:
		\begin{align}
						&\Lambda_{\mu,d}(\x_0)N(d) - f(\x_0)\nonumber\\
						\leq\quad& \left(f(\x_0 ) + \epsilon_d + 2^{2 - d^{\alpha-1/2}/2}\sup_{\z \in Z}\left( \Lambda_{\umeas,d - \lfloor d^{\alpha}\rfloor}(\z)  \right)^{-1}\right)  \Lambda_{\umeas,d - \lfloor d^{\alpha}\rfloor}(\x_0)N(d) - f(\x_0)\nonumber\\
						=\quad& \left(\epsilon_d + 2^{2 - d^{\alpha-1/2}/2}\sup_{\z \in Z} \left(\Lambda_{\umeas,d - \lfloor d^{\alpha}\rfloor}(\z)  \right)^{-1}\right)  \Lambda_{\umeas,d - \lfloor d^{\alpha}\rfloor}(\x_0) N(d - \lfloor d^{\alpha}\rfloor)\frac{N(d)}{N(d - \lfloor d^{\alpha}\rfloor)} \nonumber\\
						& + f(\x_0) \left(  \Lambda_{\umeas,d - \lfloor d^{\alpha}\rfloor}(\x_0) N(d - \lfloor d^{\alpha}\rfloor)\frac{N(d)}{N(d - \lfloor d^{\alpha}\rfloor)}- 1 \right)\nonumber\\
						\leq\quad& \left(\epsilon_d + 2^{2 - d^{\alpha-1/2}/2}\sup_{\z \in Z}\left( \Lambda_{\umeas,d - \lfloor d^{\alpha}\rfloor}(\z)  \right)^{-1}\right)  \Lambda_{\umeas,d - \lfloor d^{\alpha}\rfloor}(\x_0) N(d - \lfloor d^{\alpha}\rfloor)\frac{N(d)}{N(d - \lfloor d^{\alpha}\rfloor)} \nonumber\\
						& + C \left|  \Lambda_{\umeas,d - \lfloor d^{\alpha}\rfloor}(\x_0) N(d - \lfloor d^{\alpha}\rfloor)\frac{N(d)}{N(d - \lfloor d^{\alpha}\rfloor)}- 1 \right|\nonumber\\
						\label{eq:asymp3}
		\end{align}
		where $C = \sup_{\x \in Z} f(\x)$. Now the upper bound in (\ref{eq:asymp3}) does not depend on $\x_0$ and goes to $0$ as $d \to \infty$ since as $d \to \infty$,
		\begin{align*}
						&\epsilon_d \to 0,\\
						&2^{2 - d^{\alpha-1/2}/2}\sup_{\z \in Z}\left( \Lambda_{\umeas,d - \lfloor d^{\alpha}\rfloor}(\z) \right)^{-1} \sim 
						2^{2 - d^{\alpha - 1/2}/2} N(d - \lfloor d^{\alpha}\rfloor)  \to 0,\\
						&\Lambda_{\umeas,d - \lfloor d^{\alpha}\rfloor}(\x_0) N(d - \lfloor d^{\alpha}\rfloor)\frac{N(d)}{N(d - \lfloor d^{\alpha} \rfloor)} \to 1,
	\end{align*}
	where the first limit is obtained by uniform continuity of $f$, the second limit comes from $N$ being a polynomial (by Assumption \ref{ass:densityEstimation}) and $\alpha > 1/2$, and the last one also follows from Assumption \ref{ass:densityEstimation} and the fact that $N$ is polynomial and $\alpha < 1$. As a result:
		
		\begin{align}
						\limsup_{d \to \infty} \sup_{\x \in Z} \Lambda_{\mu,d}(\x)N(d) - f(\x) \leq 0,
						\label{eq:asymp4}
		\end{align}
		which concludes the first part of the proof.
		
		\paragraph{The other direction:}
		To obtain the opposite direction inequality, we permute the role of $\mu$ and $\umeas$ which corresponds to a density $\tilde{f} = 1/f$ which remains positive and continuous on $Z$. We repeat similar arguments, 
		fix an arbitrary $d \in \NN^*$ and set
		\begin{align}
						\epsilon_d = \sup_{\x,\y \in Z,\,\|\x- \y\| \leq D / \sqrt{d}} |1/f(\x) - 1/f(\y)|\,,
						\label{eq:asymp5}
		\end{align}
		which again is well-defined   because $f$ is positive and continuous on the compact set $Z$ and so $1/f$ is uniformly continuous on $Z$, which ensures that the supremum is finite. Furthermore, $\epsilon_d \to 0$ as $d \to \infty$.
		
		Fix an arbitrary $\x_0 \in Z$ and the same $\alpha \in (1/2,1)$ as in the first part of the proof. Choose $P_{d}$ realizing the infimum for $\Lambda_{\mu,d }(\x_0)$. The polynomial $Q_{2\lfloor d^{\alpha}/2\rfloor} \colon \y \mapsto Q((\y-\x_0)/D)$ is the same as in the first part of the proof.
		The polynomial $P = P_{d}Q_{2\lfloor d^{\alpha}/2\rfloor}$ is of degree at most $d + \lfloor d^{\alpha}\rfloor$ and satisfies 
				\begin{align}
								\begin{cases}
									P(\x_0) = 1\\
									|P| \leq |P_{d}| \text{ on } Z\\
									P^2 \leq 2^{2 - d^{\alpha - 1/2}/2} \sup_{\x \in Z} P_{d}(\x)^2 \text{ on } Z \setminus B_{D/\sqrt{d}}(\x_0).
								\end{cases}
								\label{eq:asymp6}
				\end{align}
				As $P$ is feasible one may compute an upper bound for the Christoffel function associated with $\umeas$, by:
		\begin{align*}
						&\Lambda_{\umeas,d + \lfloor d^{\alpha}\rfloor}(\x_0) \\
						\overset{(i)}{\leq} \quad & \int_{Z} (P(\x))^2 d\umeas(\x)\\
						\overset{(ii)}{\leq} \quad & \int_{Z \cap \B(\x_0, D/\sqrt{d})} \frac{(P_{d}(\x))^2}{f(\x)} d\mu(\x) + 2^{2 - d^{\alpha - 1/2}/2} \sup_{\x \in Z} P_{d}(\x)^2 \int_{Z \setminus \B(\x_0, D/\sqrt{d})} d\umeas\\
						\overset{(iii)}{\leq} \quad & (1/f(\x_0 ) + \epsilon_d) \int_{Z \cap \B(\x_0, D/\sqrt{d})} (P_{d}(\x))^2 d\mu(\x) + 2^{2 - d^{\alpha - 1/2}/2} \sup_{\z \in Z}\left( \Lambda_{\mu,d}(\z)  \right)^{-1} \int_{Z} (P_{d}(\x))^2 d\mu(\x)\\
						\overset{(iv)}{\leq} \quad & \left(1/f(\x_0 ) + \epsilon_d + 2^{2 - d^{\alpha - 1/2}/2}\sup_{\z \in Z} \left(\Lambda_{\mu,d }(\z)  \right)^{-1}\right) \int_{Z} (P_{d}(\x))^2 d\mu(\x) \\
						= \quad & \left(1/f(\x_0 ) + \epsilon_d + 2^{2 - d^{\alpha - 1/2}/2}\sup_{\z \in Z}\left( \Lambda_{\mu,d}(\z)  \right)^{-1}\right)  \Lambda_{\mu,d}(\x_0).
		\end{align*}
		The inequality $(i)$ follows $P(\x_0) = 1$, $(ii)$ follows by decomposition of the integral over two domains and the uniform bounds in (\ref{eq:asymp2}), $(iii)$ follows by combining (\ref{eq:asymp5}), Lemma \ref{lem:firstProperties} and the fact that $\umeas$ is a probability measure on $Z$, $(iv)$ follows by extending the domain of the first integral and the last identity is due to the choice of $P_{d}$ and Lemma \ref{lem:firstProperties}. 

		Hence we have, 
		\begin{align}
						&\frac{1}{\Lambda_{\mu,d}(\x_0)N(d)} - \frac{1}{f(\x_0)}\nonumber\\
						\leq\quad& \left(\frac{1}{f(\x_0 )} + \epsilon_d + 2^{2 - d^{\alpha-1/2}/2}\sup_{\z \in Z}\left( \Lambda_{\mu,d}(\z)  \right)^{-1}\right)  \frac{1}{\Lambda_{\umeas,d + \lfloor d^{\alpha}\rfloor}(\x_0)N(d)} - \frac{1}{f(\x_0)}\nonumber\\
						=\quad& \left(\epsilon_d + 2^{2 - d^{\alpha-1/2}/2}\sup_{\z \in Z}\left( \Lambda_{\mu,d}(\z)  \right)^{-1}\right)  \frac{1}{\Lambda_{\umeas,d + \lfloor d^{\alpha}\rfloor}(\x_0) N(d + \lfloor d^{\alpha}\rfloor)}\frac{N(d + \lfloor d^{\alpha}\rfloor)}{N(d )} \nonumber\\
						& + \frac{1}{f(\x_0)} \left( \frac{1}{\Lambda_{\umeas,d + \lfloor d^{\alpha}\rfloor}(\x_0) N(d + \lfloor d^{\alpha}\rfloor)}\frac{N(d + \lfloor d^{\alpha}\rfloor)}{N(d )} - 1 \right)\nonumber\\
						\leq\quad& \left(\epsilon_d + C_22^{2 - d^{\alpha -1/2}/2}\sup_{\z \in Z}\left( \Lambda_{\umeas,d}(\z)  \right)^{-1}\right)\frac{1}{\Lambda_{\umeas,d + \lfloor d^{\alpha}\rfloor}(\x_0) N(d + \lfloor d^{\alpha}\rfloor)}\frac{N(d + \lfloor d^{\alpha}\rfloor)}{N(d )} \nonumber\\
						& + C_2 \left| \frac{1}{\Lambda_{\umeas,d + \lfloor d^{\alpha}\rfloor}(\x_0) N(d + \lfloor d^{\alpha}\rfloor)}\frac{N(d + \lfloor d^{\alpha}\rfloor)}{N(d )}- 1 \right|\nonumber\\
						\label{eq:asymp7}
		\end{align}
		where $C_2 = \sup_{\z \in Z} 1/f(\z)$. For the last inequality, we have used the fact $\mu$ has density $f$ with respect to $\lambda$, so that for all $\x \in Z$, $\inf_{\z \in Z} f(\z) \Lambda_{\umeas,d}(\x) = \Lambda_{\umeas,d}(\x) / C_2 \leq \Lambda_{\mu,d}(\x)$ and hence $\sup_{\z \in Z}\left( \Lambda_{\mu,d}(\z)  \right)^{-1} \leq C_2 \sup_{\z \in Z}\left( \Lambda_{\umeas,d}(\z)  \right)^{-1} $. Now the upper bound in (\ref{eq:asymp7}) does not depend on $\x_0$ and goes to $0$ as $d \to \infty$ since as $d \to \infty$,
		\begin{align*}
						&\epsilon_d \to 0,\\
						&2^{2 - d^{\alpha - 1/2}/2}\sup_{\z \in Z}\left( \Lambda_{\umeas,d }(\z) \right)^{-1} \sim 
						2^{2 - d^{\alpha - 1/2}/2} N(d )  \to 0,\\
						& \frac{1}{\Lambda_{\umeas,d + \lfloor d^{\alpha}\rfloor}(\x_0) N(d + \lfloor d^{\alpha}\rfloor)}\frac{N(d + \lfloor d^{\alpha}\rfloor)}{N(d )}\to 1.
	\end{align*}
	The first limit is obtained by uniform continuity of $1/f$, the second limit follows from $N$ being polynomial (by Assumption \ref{ass:densityEstimation}) and $\alpha > 1/2$, and the last one also follows from Assumption \ref{ass:densityEstimation}$N$ being a polynomial and $\alpha < 1$. Therefore:
		\begin{align}
						\limsup_{d \to \infty} \sup_{\x \in Z} \frac{1}{\Lambda_{\mu,d}(\x)N(d)} - \frac{1}{f(\x)} \leq 0,
						\label{eq:asymp8}
		\end{align}
		from which we deduce that 	
		\begin{align}
						\liminf_{d \to \infty} \inf_{\x \in Z} \Lambda_{\mu,d}(\x)N(d) - f(\x) \geq 0.
						\label{eq:asymp9}
		\end{align}
		Combining (\ref{eq:asymp4}) and (\ref{eq:asymp9}) concludes the proof.
\end{proof}
\begin{remark}[Proof of the convergence rate]
				In the case of the sphere, Assumption \ref{ass:densityEstimation} does not hold only in the limit but for all $d$ and 
				from the proof of Theorem \ref{th:asymptoticsUnitSphere} we are able
				to obtain  explicit quantitative bounds and a convergence rate. Suppose $f \colon \RR^p \to \RR_+$ with $0 < c \leq f \leq C$ on $\mathbb{S}^{p-1}$ and that $\mu$ has density $f$ with respect to $\lambda$, the uniform measure on the sphere. In this case, for all $d \in \NN$, $\Lambda_{\lambda,d} = 1/N(d)$, since, by isotropy $\Lambda_{\lambda,d}$ is constant on the sphere with $\int (\Lambda_{\lambda,d})^{-1}d\lambda = N(d)$. In addition, assume that both $f$ and $1/f$ are $L$-Lipschitz on the unit ball for some $L > 0$. Then we can choose in both cases of the proof $\epsilon_d \leq LD / \sqrt{d}$ and equation (\ref{eq:asymp3}) simplifies to:
				\begin{align*}
								&\Lambda_{\mu,d}(\x_0)N(d) - f(\x_0)\\
								\leq \quad& \left(\frac{2LD}{\sqrt{d}} + 2^{2 - d^{\alpha - 1/2}/2}N(d - \lfloor d^{\alpha}\rfloor)\right) \frac{N(d)}{N(d - \lfloor d^{\alpha}\rfloor)}
								+ C \left(  \frac{N(d)}{N(d - \lfloor d^{\alpha}\rfloor)}- 1 \right)\nonumber.\\
				\end{align*}
				Similary and equation (\ref{eq:asymp7}) yields
				\begin{align*}
								&\Lambda_{\mu,d}(\x_0)N(d) -f(\x_0)\\
								\geq \quad &- \Lambda_{\mu,d(\x_0)} N(d)f(\x_0) \left(\left(\frac{2LD}{\sqrt{d}} + \frac{2^{2 - d^{\alpha - 1/2}/2}}{c}N(d)\right)\frac{N(d + \lfloor d^{\alpha}\rfloor)}{N(d )} +\frac{1}{c} \left( \frac{N(d + \lfloor d^{\alpha}\rfloor)}{N(d )}- 1 \right)\right).\nonumber\\
								\geq \quad &- C^2 \left(\left(\frac{2LD}{\sqrt{d}} + \frac{2^{2 - d^{\alpha - 1/2}/2}}{c}N(d)\right)\frac{N(d + \lfloor d^{\alpha}\rfloor)}{N(d )} +\frac{1}{c} \left( \frac{N(d + \lfloor d^{\alpha}\rfloor)}{N(d )}- 1 \right)\right).\nonumber
				\end{align*}
				In both expressions, the last term is the leading term of order $d^{\alpha - 1}$ where $-1/2 < \alpha -1 < 0$.
				As a result we obtain an overall convergence rate of order $O(d^{\alpha - 1})$, which can be arbitrarily close $O(d^{-1/2})$, which to the best of our knowledge is the first estimate of this type in this context. We leave for future research the task of improving this rate. Note that in the the special case of Lipschitz densities, in the proof we need $ d^{\alpha}/\sqrt{d} \to +\infty$. We choose a decrease of order $1/\sqrt{d}$ to find the denominator. A slower decrease would lead to more slack for the choice of $\alpha$ but then the first term in both expressions in our explicit bounds becomes dominant of order bigger than $O(d^{-1/2})$. The choice of decrease of order $1/\sqrt{d}$ seems to achieve balance between the two terms using the proposed proof technique.
				\label{rem:explicitSphere}
\end{remark}

\subsection{Perturbation of the moment matrix}
\begin{proof}[Proof of Lemma \ref{lem:perturb}:]
			Since $A$ is positive semidefinite, the left-hand side must be greater than the right-hand side. In addition, the infimum in the right-hand side is attained at some $\x^*$
			which can be used in the right-hand side to show the result.
\end{proof}

\section*{Acknowledgements}
The research of J.B. Lasserre was funded by the European Research Council (ERC) under the European Union Horizon 2020 research and innovation program (grant agreement 666981 TAMING). This work was partially supported by CIMI (Centre International de Math\'ematiques et d'Informatique). Edouard Pauwels acknowledges the support of the French Agence Nationale de la Recherche (ANR) under references ANR-PRC-CE23 Masdol under grant ANR-PRC-CE23 and ANR-3IA Artificial and Natural Intelligence Toulouse Institute, the support of Air Force Office of Scientific Research, Air Force Material Command, USAF, under grant number FA9550-18-1-0226. Mihai Putinar is grateful to LAAS and Universite Paul Sabatier for support. Jean-Bernard Lasserre and Edouard Pauwels would like to thank Yousouf Emin for early discussions on the topic.

\appendix
\section{Introduction}
We collect below some supplementary information to the main body of the present article. Section \ref{sec:notations} provides additional details regarding the notations and constructions employed in the main text. More insights about the numerical simulations are given in Section \ref{sec:numericsAdd}. Technical Lemmas are presented in Section \ref{sec:lemmas}.
\section{Polynomials and the moment matrix}
\label{sec:notations}
In this section $\mu$ denotes a positive Borel measure on $\RR^p$ with finite moments and for any $d \in \NN$, $\MM_{\mu,d}$ denotes its moment matrix with moments of degree up to $2d$. As a matter of fact we will soon restrict our attention to probability measures, which is a minor constraint to impose on the constructs below.

Henceforth we fix the dimension of the ambient euclidean space to be $p$. For example, vectors in $\RR^p$ as well as $p$-variate polynomials with real coefficients. 
We denote by $X$ the tuple of $p$ variables $X_1, \ldots, X_p$ which appear in mathematical expressions involving polynomials. Monomials from the canonical basis of $p$-variate polynomials are identified with their exponents in $\NN^p$: specifically  $\alpha = (\alpha_i)_{i =1 \ldots p} 
\in \NN^p$ is associated to the  monomial $X^\alpha := X_1^{\alpha_1} X_2^{\alpha_2} \ldots X_p^{\alpha_p}$ of degree 
$\deg(\alpha) := \sum_{i=1}^p \alpha_i=\vert\alpha\vert$. The notations $\lgl$ and $\leqgl$ stand for the graded lexicographic order, a well ordering over $p$-variate monomials. This amounts to, first, use the canonical order on the degree and, second, break ties in monomials with the same degree using the lexicographic order with $X_1 =a, X_2 = b \ldots$ For example, the monomials in two variables $X_1, X_2$, of degree less than or equal to $3$ listed in this order are given by: $1,\,X_1, \,X_2, \,X_1^2, \,X_1X_2, \,X_2^2,\, X_1^3,\, X_1^2X_2,\, X_1X_2^2,\, X_2^3$. We focus here on the graded lexicographic order to provide a concrete example, but any ordering compatible with the degree would work similarly. 

By definition $\NN^p_d$ is the set $\left\{ \alpha \in \NN^p;\; \deg(\alpha) \leq d \right\}$, while $\RR[X]$  is the algebra of $p$-variate polynomials with real coefficients. The degree of a polynomial is the highest of the degrees of its monomials with nonzero coefficients\footnote{For the null polynomial, we use the convention that its degree is $-\infty$.}. The notation $\deg(\cdot)$ applies a polynomial as well as to an element of $\NN^p$. For $d \in \NN$, $\RR_d[X]$ stands for the set of $p$-variate polynomials of degree less than or equal to $d$. We set $s(d) = {p+d \choose d} = \dim \RR_d[X]$; this is of course the number of monomials of degree less than or equal to $d$. 

From now on  $\v_d(X)$ denotes the vector of monomials of degree less or equal to $d$ (sorted using $\leqgl$), i.e., $\v_d(X) := \left( X^\alpha \right)_{\alpha \in \NN^p_d} \in \RR_d[X]^{s(d)}$. With this notation, one can write a polynomial $P\in\RR_d[X]$ as $P(X) = \left\langle\p, \v_d(X)\right\rangle$ for some real vector of coefficients $\p = \left( p_{\alpha} \right)_{\alpha \in \NN_d^p} \in \RR^{s(d)}$ ordered using $\leqgl$. 
Given $\x = (x_i)_{i = 1 \ldots p} \in \RR^p$, $P(\x)$ denotes the evaluation of $P$ with respect to the assignments $X_1 = x_1, X_2 = x_2, \ldots X_p = x_{p}$. Given a Borel probability measure $\mu$ and $\alpha \in \NN^p$, $y_{\alpha}(\mu)$ denotes the moment $\alpha$ of $\mu$, i.e., $y_{\alpha}(\mu) = \int_{\RR^p} \x^{\alpha} d\mu(\x)$. 
Throughout the paper we will only consider rapidly decaying measures, that is measure whose all moments are finite.
For a positive Borel measure $\mu$ on $\R^p$ denote by ${\rm supp}(\mu)$ its support, i.e., the smallest closed set $\om\subset\R^p$ such that $\mu(\R^p\setminus\om)=0$.

\paragraph{Moment matrix}
The moment matrix of $\mu$, $\MM_{\mu,d}$, is a matrix indexed by monomials of degree at most $d$ ordered with respect to $\leqgl$. For $\alpha,\beta \in \NN^p_d$, the corresponding entry in $\MM_{\mu,d}$ is defined by $[\MM_{\mu,d}]_{\alpha,\beta} := y_{\alpha +\beta}(\mu)$, the moment 
$\int\x^{\alpha+\beta}d\mu$ of $\mu$. For example, in the case $p = 2$, letting $y_{\alpha} = y_{\alpha} (\mu)$ for $\alpha \in \NN_4^2$, one finds:
$$\MM_{\mu,2}: \quad\begin{array}{rccccccccc}
& & &1& X_1  & X_2  & X_1^2 & X_1 X_2 &  X_2^2 \\
 & & & \\
 1  &   \quad & &1 & y_{10} & y_{01} & y_{20}& y_{11}& y_{02} \\
X_1&   \quad &&y_{10} & y_{20} & y_{11}&y_{30} & y_{21}& y_{12} \\
X_2&   \quad & &y_{01} & y_{11} & y_{02}& y_{21} &y_{12}& y_{03} \\
X_1^2& \quad & &y_{20} & y_{30} & y_{21}& y_{40} & y_{31}& y_{22} \\
X_1X_2& \quad & &y_{11} & y_{21} & y_{12}& y_{31} & y_{22}& y_{13}\\
X_2^2& \quad & &y_{02} & y_{12} &y_{03}& y_{22} & y_{13}& y_{04}\\
\end{array}.$$
The matrix $\MM_{\mu,d}$ is positive semidefinite for all $d \in \NN$. Indeed, for any $p \in \RR^{s(d)}$, let $P \in \RR_d[X]$ be the polynomial with vector of coefficients $\p$; then $\p^T\MM_{\mu,d}\p = \int_{\RR^p} P(\x)^2 d\mu(\x) \geq 0$. We also have the identity $\MM_{\mu,d} = \int_{\RR^p} \v_d(\x) \v_d(\x)^T d\mu(\x)$ where the integral is understood element-wise. Actually it is useful to interpret the moment matrix as representing the bilinear form
\begin{align*}
				\left\langle \cdot, \cdot\right\rangle_\mu \colon \RR[X] \times \RR[X]& \mapsto \RR \\
				(P,Q)& \mapsto \int_{\RR^p} P(\x)Q(\x)d\mu(\x),
\end{align*}
restricted to polynomials of degree up to $d$. Indeed, if $\p,\q \in \RR^{s(d)}$ are the vectors of coefficients of any two polynomials $P$ and $Q$ of degree up to $d$, one has $\p^T \MM_{\mu,d} \q =  \left\langle P, Q\right\rangle_\mu$

\section{Numerical experiments}
\label{sec:numericsAdd}
This section provides additional details regarding numerical experiments.
\subsection{Rank of the moment matrix}
We choose four different subsets of $\RR^3$:
\begin{itemize}
				\item The unit cube $\left\{ x,y,z, |x| \leq 1, |y| \leq 1, |z|\leq 1 \right\}$.
				\item The 3 dimensional unit sphere $\left\{ x,y,z, x^2 + y^2 + z^2 =  1 \right\}$.
				\item The 3 dimensional TV screen $\left\{ x,y,z, x^6 + y^6 + z^6 - 2x^2y^2z^2 =  1 \right\}$.
				\item The 3 dimensional torus $\left\{ x,y,z, \left(x^2 + y^2 + z^2 + \frac{9}{16} - \frac{1}{16}\right)^2 - \frac{9}{16} (x^2 + y^2) =  0 \right\}$.
\end{itemize}
Among the above sets the first one is three dimensional while all the others are 
$2$-dimensional. The $2$-dimensional sets are displayed in Figure \ref{fig:rank}.
For each set, we sample 20000 points on it and compute the rank of the empirical moment matrix for different values of the degree. To perform this computation, we threshold the singular values of the design matrix consisting in the expansion of each data point in the multivariate Tchebychev polynomial basis.

\subsection{Density estimation on algebraic sets}
The following quantities are used in the litterature as divergences combined with density estimation techniques for the examples treated in the main article.
\begin{itemize}
				\item The quantity $\cos(\theta_1 - \theta_2)$ where $\theta_1$ and $\theta_2$ are angular coordinates of two points on the circle.
				\item The dot product on the sphere which generalizes the previous situation to larger dimensions, used in \cite{cuevas2006plugin}.
				\item The quantity $\cos\left( \sqrt{(\phi_1 - \phi_2)^2+(\psi_1 - \psi_2)^2 }  \right)$ where $\phi_1, \phi_2, \psi_1, \psi_2$ are angles which correspond to points on the torus in $\RR^4$, used in \cite{lovell2003structure}.
\end{itemize}

\section{Technical Lemmas}
\label{sec:lemmas}
We begin with the following simple Lemma.
\begin{lemma}
				Let $\mu$ be a Borel probability measure on $\RR^p$ and $S$ its support which is assumed to be a bounded subset of $\RR^p$.
				Then for any $d \in \NN^*$ and for any $P \in \RR_d[X]$,
				\begin{align*}
								\sup_{\x \in S} |P(\x)|^2 \leq \sup_{\z \in S} \Lambda_{\mu,d}^{-1}(\z) \int (P(\x))^2 d\mu(\x).
				\end{align*}
				\label{lem:firstProperties}
\end{lemma}
\begin{proof}
				For any $\x \in S$ and $P \in \RR_d[X]$, one finds
				\begin{align*}
								\Lambda_{\mu,d}(\x) \leq \frac{\int (P(\z))^2 d \mu(\z)}{P(\x)^2}
				\end{align*}
				and
				\begin{align*}
								P(\x)^2 \leq \frac{\int (P(\z))^2 d \mu(\z)}{\Lambda_{\mu,d}(\x)}.
				\end{align*}
				The result follows by considering the supremum over $S$ on both sides.
\end{proof}
Finally, the following Lemma is a quantitative adaptation of \cite[Lemma 2.1]{kroo2012christoffel}.
\begin{lemma}	
	\label{lem:needlePoly}
	For any $d \in \NN^*$ and any $\delta \in (0,1)$, there exists a $p$-variate polynomial of degree $2d$, $Q$, such that
\[Q(0) \,=\, 1\,;\quad 	-1\,\leq \,Q \,\leq\, 1,\text{ on } \B\,;\quad 
		\vert Q\vert\,\leq\, 2^{1-\delta d} \text{ on } \B\setminus \B_{\delta}(0).\]
\end{lemma}
\begin{proof}
	Let $R$ be the univariate polynomial of degree $2d$, defined by
	\begin{align*}
		R\colon t \to \frac{T_d(1+\delta^2 - t^2)}{T_d(1+\delta^2)},
	\end{align*}
	where $T_d$ is the Chebyshev polynomial of the first kind. We obtain
	\begin{align}
		\label{eq:needlePolyTemp1}
		R(0) = 1. 
	\end{align}
	Furthermore, for $t \in [-1,1]$, we have $0 \leq 1+\delta^2 - t^2 \leq 1+\delta^2$. $T_d$ has absolute value less than $1$ on $[-1,1]$ and is inceasing on $[1, \infty)$ with $T_d(1) = 1$, so for $t \in [-1,1]$,
	\begin{align}
		\label{eq:needlePolyTemp2}
		-1 \leq R(t) \leq 1.
	\end{align}
	For $|t| \in [\delta, 1]$, we have $\delta^2 \leq 1+\delta^2-t^2 \leq 1$, so
	\begin{align}
		\label{eq:needlePolyTemp3}
		|R(t)| \leq \frac{1}{T_d(1+\delta^2)}.
	\end{align}
	Let us bound the last quantity. Recall that for $t \geq 1$, we have the following explicit expresion 
	\begin{align*}
		T_d(t) = \frac{1}{2}\left( \left( t + \sqrt{t^2-1} \right)^d +  \left( t + \sqrt{t^2-1} \right)^{-d}\right).
	\end{align*}
	We have $1 + \delta^2 +\sqrt{(1+\delta^2)^2 - 1} \geq 1 + \sqrt{2} \delta$, which leads to
	\begin{align}
		\label{eq:needlePolyTemp4}
		T_d(1+\delta^2) &\geq \frac{1}{2}\left( 1 + \sqrt{2} \delta \right)^d\\
		&= \frac{1}{2} \exp\left( \log\left( 1+\sqrt{2}\delta \right) d \right)\nonumber\\
		&\geq\frac{1}{2} \exp\left( \log(1+\sqrt{2}) \delta d \right)\nonumber\\
		&\geq2^{\delta d - 1}, \nonumber
	\end{align}
	where we have used concavity of the $\log$ and the fact that $1+\sqrt{2} \geq 2$. It follows by combining (\ref{eq:needlePolyTemp1}), (\ref{eq:needlePolyTemp2}), (\ref{eq:needlePolyTemp3}) and (\ref{eq:needlePolyTemp4}), that $Q \colon \y \to R(\|\y\|_2)$ satisfies the claimed properties.
\end{proof}

\subsection{Proof of Lemma \ref{lem:irreducibleVanish}}
\begin{proof}
				The implication (ii) to (i) is trivial since $\sigma_W$ is supported on $W$. Let us assume that (i) is true and deduce (ii). This is classicaly formulated in the language of sheafs, we adopt a more elementary language and describe details for completeness. The main idea of the proof is to reduce the statement to R\"uckert's complex analytic Nullstellensatz \cite[Proposition 1.1.29]{huybrechts2006complex} which characterizes the class of analytic functions vanishing locally on the zero set of a family of analytic functions. The first point is precisely of local nature and the Nullstellensatz combined with properties of analytic functions allow to deduce properties of $P$ which extend globally \cite{serre1956geometrie}. The irreducibility hypothesis and the definition of the area measure allow precisely to switch from real to complex variables.

				First, let $P_1, \ldots, P_k$ generate the ideal $I$ of polynomials vanishing on $W$. Since $W$ is irreducible, $I$ is a prime ideal and $W = \left\{ \x \in \RR^p,\, P_i(\x) = 0,\, i =1,\ldots,k \right\}$ (see Propositions 3.3.14 and 3.3.16 in \cite{bochnak1998real}). Point (i) entails that there exists $\x_0 \in W$ and $U_1$ an Euclidean neighborhood of $\x_0$ in $\RR^p$ such that $W \cap U_1$ is an analytic submanifold, or more precisely the Jacobian matrix $\left( \frac{\partial P_i}{\partial X_j} \right)_{i=1\ldots k,\, j=1\ldots p}$ has rank $k$ on $U_1$, and furthermore, $P(\x) = 0$ for all $\x \in W \cap U_1$.

				Consider the complex analytic manifold $Z = \left\{ \z \in \CC^p,\, P_i(\z) = 0,\, i =1,\ldots,k \right\}$ and the polynomial map
				\begin{align*}
						G\colon (X_1,\ldots,X_p) \mapsto (P_1(X_1,\ldots,X_p),\ldots,P_k(X_1,\ldots,X_p),X_{k+1},\ldots, X_p).
				\end{align*}
				This map is locally invertible around $\x_0$ in $\CC^p$ and its inverse is analytic. The function $(H\colon X_{k+1},\ldots,X_p)\mapsto P(G^{-1}(0,\ldots,0,X_{k+1},\ldots,X_p)$ is analytic and vanishes in an Euclidean neighborhood, of $(x_{0,k+1},\ldots,x_{0,p})$, the last $k$ coordinates of $\x_0$, in $\RR^k$. Hence, it can be identified with the constant null function on $U_1$. This shows that $H$ vanishes in an Euclidean neighborhood of $(x_{0,k+1},\ldots,x_{0,p})$ in $\CC^k$.

								This proves that there exists a Euclidean neighborhood of $\x_0$ in $\CC^p$, $U_2$, such that $P$ vanishes on $Z \cap U_2$. At this point we can invoke the complex analytic Nullstellensatz \cite[Proposition 1.1.29]{huybrechts2006complex}, to obtain $k$ analytic functions, $O_1,\ldots,O_k$, on $U_2$, and an integer $m \geq 1$, such that, for all $\z \in U_2$, we have
				\begin{align}
								(P(\z))^m = \sum_{i=1}^k P_i(\z) O_i(\z).
								\label{eq:radicalMembership}
				\end{align}
				We can now use powerful result related to completion of local rings. Combining Corollary 1 of Proposition 4 and Proposition 22 in \cite{serre1956geometrie} we obtain that identity \eqref{eq:radicalMembership} still holds with the constraint that $O_i$, $i=1,\ldots, k$, are rational functions. In other words, reducing to common denominator, there exists an Euclidean neighobourhood of $\x_0$, $U_3$, and $k$ complex polynomials $Q_1,\ldots,Q_k$ and a complex polynomial $Q$ such that $Q$ does not vanish on $U_3$, and
				\begin{align}
								Q(\z)(P(\z))^m = \sum_{i=1}^k P_i(\z) Q_i(\z).
								\label{eq:radicalMembershipPoly}
				\end{align}
				From this we deduce that identity \eqref{eq:radicalMembershipPoly} still holds by restricting to real variables and real polynomials. Now since the ideal $I$ is prime and $Q \not\in I$ (since $Q(\x_0) \neq 0$), we deduce that $P \in I$, that is, $P$ vanishes on $W$. This is what we wanted to prove.
\end{proof}

\end{document}